%% file: main.tex
\documentclass{article}

\input{math_commands.tex}

\usepackage{arxiv}
\usepackage{hyperref}
\usepackage{url}
\usepackage{thmtools}
\usepackage{amssymb,amsmath,amsthm}
\usepackage[scr=boondoxo,scrscaled=1.05]{mathalfa}
\usepackage[pdftex]{graphicx} 
\usepackage{wrapfig}
\usepackage{subcaption}
\usepackage[authoryear]{natbib}
\title{The Convex Information Bottleneck Lagrangian}

\author{Borja Rodr\'iguez G\'alvez, Ragnar Thobaben and Mikael Skoglund \\
Department of Intelligent Systems\\
Division of Information Science and Engineering (ISE) \\
KTH Royal Institute of Technology\\
100 44 Stockholm, Sweden \\
\texttt{\{borjarg,ragnart,skoglund\}@kth.se} 
}

\newtheoremstyle{remboldstyle}
   {}{}{\itshape}{}{\bfseries}{.}{.5em}{{\thmname{#1 }}{\thmnumber{#2}}{\thmnote{ (#3)}}}
\theoremstyle{remboldstyle}

\newtheorem{Theorem}{Theorem}

\newtheorem{Corollary}{Corollary}
\newtheorem{Proposition}{Proposition}
\newtheorem{Definition}{Definition}
\newtheorem{Lemma}{Lemma}
\newtheorem{Remark}{Remark}

\DeclareMathOperator*{\argsup}{arg\,sup}
\DeclareMathOperator*{\arginf}{arg\,inf}

\begin{document}

\maketitle

\begin{abstract}
The information bottleneck (IB) problem tackles the issue of obtaining relevant compressed representations $T$ of some random variable $X$ for the task of predicting $Y$. It is defined as a constrained optimization problem which maximizes the information the representation has about the task, $I(T;Y)$, while ensuring that a certain level of compression $r$ is achieved (i.e., $ I(X;T) \leq r$). For practical reasons, the problem is usually solved by maximizing the IB Lagrangian (i.e., $\mathcal{L}_{\text{IB}}(T;\beta) = I(T;Y) - \beta I(X;T)$) for many values of $\beta \in [0,1]$. Then, the curve of maximal $I(T;Y)$ for a given $I(X;T)$ is drawn and a representation with the desired predictability and compression is selected. It is known when $Y$ is a deterministic function of $X$, the IB curve cannot be explored and {another Lagrangian has} been proposed to tackle this problem: the squared IB Lagrangian: $\mathcal{L}_{\text{sq-IB}}(T;\beta_{\text{sq}})=I(T;Y)-\beta_{\text{sq}}I(X;T)^2$. In this paper{,} we (i) present a general family of Lagrangians which allow for the exploration of the IB curve in all scenarios; (ii) provide the exact one-to-one mapping between the Lagrange multiplier and the desired compression rate $r$ for known IB curve shapes{; and (iii) show we can approximately obtain a specific compression level with the convex IB Lagrangian for both known and unknown IB curve shapes}. This eliminates the burden of solving the optimization problem for many values of the Lagrange multiplier. That is, we prove that we can solve the original constrained problem with a single optimization.
\end{abstract}

\section{Introduction}
\label{sec:introduction}

Let $X \in \mathcal{X}$ and $Y \in \mathcal{Y}$ be two statistically dependent random variables with joint distribution $p_{(X,Y)}$. The information bottleneck (IB) \citep{tishby2000information} investigates the problem of extracting the relevant information from $X$ for the task of predicting $Y$. 

For this purpose, the IB defines a bottleneck variable $T \in \mathcal{T}$ obeying the Markov chain $Y \leftrightarrow X \leftrightarrow T$ so that $T$ acts as a representation of $X$. \citet{tishby2000information} define the relevant information as the information the representation keeps from $Y$ after the compression of $X$ (i.e., $I(T;Y)$), provided a certain level of compression (i.e,  $I(X;T) \leq r$). Therefore, we select the representation which yields the value of the IB curve that best fits our requirements.

\begin{Definition}[IB functional]
\label{def:ib_functional}
Let $X$ and $Y$ be statistically dependent variables. Let $\Delta$ be the set of random variables $T$ obeying the Markov condition $Y \leftrightarrow X \leftrightarrow T$. Then the IB functional is 

\begin{equation}
	F_{\textnormal{IB,max}}(r) = \max_{T \in \Delta} \left\lbrace I(T;Y) \right\rbrace \text{ s.t. } I(X;T) \leq r, \ \forall r \in [0,\infty).
	\label{eq:ib_functional}
\end{equation}
\end{Definition}

\begin{Definition}[IB curve]
\label{def:ib_curve}
The IB curve is the set of points defined by the solutions of
$F_{\textnormal{IB,max}}(r)$ for varying values of $r \in [0,\infty)$.
\end{Definition} 

\begin{Definition}[Information plane] The plane is defined by the axes $I(T;Y)$ and $I(X;T)$.
\end{Definition}

{This method has been successfully applied to solve different problems from a variety of domains. For example: 
\begin{itemize}
\item \textbf{Supervised learning.} In supervised learning, we are presented with a set of $n$ pairs of input features and task outputs instances. We seek an approximation of the conditional probability distribution between the task outputs $Y$ and the input features $X$. In classification tasks (i.e., when $Y$ is a discrete random variable), the introduction of the variable $T$ learned through the information bottleneck principle maintained the performance of standard algorithms based on the cross-entropy loss while providing with more adversarial attacks robustness and invariance to nuisances \citep{alemi2016deep, peng2018variational, achille2018information}. Moreover, by the nature of its definition the information bottleneck appears to be closely related with a trade-off between accuracy on the observable set and generalization to new, unseen instances (see Section \ref{sec:ib_supervised_learning}).
\item \textbf{Clustering.} In clustering, we are presented with a set of $n$ pairs of instances of a random variable $X$ and their attributes of interest $Y$. We seek for groups of instances (or clusters $T$) such that the attributes of interest within the instances of each cluster are similar and the attributes of interest of the instances of different clusters are dissimilar. Therefore, the information bottleneck can be employed since it allows to aim for attribute representative clusters (maximizing the similarity between instances within the clusters) and enforce a certain compression of the random variable $X$ (ensuring a certain difference between instances of the different clusters). This has been successfully implemented, for instance, for gene expression analysis and word, document, stock pricing or movie rating clustering \citep{slonim2000document,slonim2000agglomerative,slonim2005information}.
\item	\textbf{Image Segmentation.} In image segmentation, we want to partition an image into segments such that each pixel in a region shares some attributes. If we divide the image into very small regions $X$ (e.g., each region is a pixel or a set of pixels defined by a grid), we can consider the problem of segmentation as that of clustering the regions $X$ based on the region attributes $Y$. Hence, we can use the information bottleneck so that we seek region clusters $T$ that are maximally informative about the attributes $Y$ (e.g., the intensity histogram bins) and maintain a level of compression of the original regions $X$ \citep{teahan2000text}.
\item \textbf{Quantization.} In quantization, we consider a random variable $X \in \mathcal{X}$ such that $\mathcal{X}$ is a large or continuous set. Our objective is to map $X$ into a variable $T \in \mathcal{T}$ such that $\mathcal{T}$ is a smaller, countable set. If we fix the quantization set size to $|\mathcal{T}| = \floor{r}$ and aim at maximizing the information of the quantized variable with another random variable $Y$ and restric the mapping to be deterministic, then the problem is equivalent to the information bottleneck \citep{strouse2017deterministic, nazer2017information}.
\item \textbf{Source coding.} In source coding, we consider a data source $\mathcal{S}$ which generates a signal $Y \in \mathcal{Y}$, which is later perturbed by a channel $\mathcal{C}: \mathcal{Y} \rightarrow \mathcal{X} $ that outputs $X$. We seek a coding scheme that generates a code $T \in \mathcal{T}$ from the output of the channel $X$ which is as informative as possible about the original source signal $Y$ and can be transmitted at a small rate $I(X;T) \leq r$. Therefore, this problem is equivalent to the the formulation of the information bottleneck \citep{hassanpour2018equivalence}.
\end{itemize}
}

{Furthermore, it has been employed as a tool for development or explanation in other disciplines like reinforcement learning \citep{goyal2019infobot,yingjun2019learning,sharma2020dynamicsaware}, attribution methods \citep{schulz2020restricting}, natural language processing \citep{li2019specializing}, linguistics \citep{zaslavsky2018efficient} or neuroscience \citep{chalk2018toward}. Moreover, it has connections with other problems such as source coding with side information (or the Wyner-Ahlswede-K\"orner (WAK) problem), the rate-distortion problem or the cost-capacity problem (see Sections 3, 6 and 7 from \citep{gilad2003information}).}

In practice, solving a constrained optimization problem such as the IB functional is {challenging}. Thus, in order to avoid the non-linear constraints from the IB functional the IB Lagrangian is defined.

\begin{Definition}[IB Lagrangian]
\label{def:ib_lagrangian}
Let $X$ and $Y$ be statistically dependent variables. Let $\Delta$ be the set of random variables $T$ obeying the Markov condition $Y \leftrightarrow X \leftrightarrow T$. Then we define the IB Lagrangian as

\begin{equation}
	\mathcal{L}_{\text{IB}}(T;\beta) = I(T;Y) - \beta I(X;T).
\end{equation}

\end{Definition}

Here $\beta \in [0,1]$ is the Lagrange multiplier which controls the trade-off between the information of $Y$ retained and the compression of $X$. Note we consider $\beta \in [0,1]$ because (i) for $\beta \leq 0$ many uncompressed solutions such as $T = X$ maximize $\mathcal{L}_{\text{IB}}(T;\beta)$, and (ii) for $\beta \geq 1$ the IB Lagrangian is non-positive due to the data processing inequality (DPI) (Theorem 2.8.1 from \citet{cover2012elements}) and trivial solutions like $T=\text{const}$ are maximizers with $\mathcal{L}_{\text{IB}}(T;\beta) = 0$ \citep{kolchinsky2018caveats}. 

We know the solutions of the IB Lagrangian optimization (if existent) are solutions of the IB functional by the Lagrange's sufficiency theorem (Theorem 5 in Appendix A of \citet{courcoubetis2003pricing}). Moreover, since the IB functional is concave (Lemma 5 of \citet{gilad2003information}) we know they exist (Theorem 6 in Appendix A of \citet{courcoubetis2003pricing}).

Therefore, the problem is usually solved by maximizing the IB Lagrangian with adaptations of the Blahut-Arimoto algorithm \citep{tishby2000information}, deterministic annealing approaches \citep{tishby2001data} or a bottom-up greedy agglomerative clustering \citep{slonim2000agglomerative} or its improved sequential counterpart \citep{slonim2002unsupervised}. However, when provided with high-dimensional random variables $X$ such as images, these algorithms do not scale well and deep learning based techniques, where the IB Lagrangian is used as the objective function, prevailed \citep{alemi2016deep, chalk2016relevant, kolchinsky2017nonlinear}.

Note the IB Lagrangian optimization yields a representation $T$ with a given performance ($I(X;T), I(T;Y)$) for a given $\beta$. However{,} there is no one-to-one mapping between $\beta$ and $I(X;T)$. Hence, we cannot directly optimize for a desired compression level $r$ but we need to perform several optimizations for different values of $\beta$ and select the representation with the desired performance; e.g., \citep{alemi2016deep}. The Lagrange multiplier selection is important since (i) sometimes even choices of $\beta < 1$ lead to trivial representations such that $p_{T|X} = p_T$, and (ii) there exist some discontinuities on the performance level w.r.t. the values of $\beta$ \citep{wu2019learnability}. 

Moreover, recently \citet{kolchinsky2018caveats} showed how in deterministic scenarios (such as many classification problems where an input $x_i$ belongs to a single {particular} class $y_i$) the IB Lagrangian could not explore the IB curve. Particularly, they showed that multiple $\beta$ yielded the same performance level and that a single value of $\beta$ could result in different performance levels. To solve this issue, they introduced the squared IB Lagrangian, $\mathcal{L}_{\text{sq-IB}}(T;\beta_{\text{sq}}) = I(T;Y) - \beta_{sq} I(X;T)^2$, which is able to explore the IB curve in any scenario by optimizing for different values of $\beta_{\text{sq}}$. However, even though they realized a one-to-one mapping between $\beta_{\text{sq}}$ and the compression level existed, they did not find such mapping. Hence, multiple optimizations of the Lagrangian were still required to find the best traded-off solution.

The main contributions of this article are:

\begin{enumerate}
	\item We introduce a general family of Lagrangians (the convex IB Lagrangians) which are able to explore the IB curve in any scenario for which the squared IB Lagrangian \citep{kolchinsky2018caveats} is a particular case of. More importantly, the analysis made for deriving this family of Lagrangians can serve as inspiration for obtaining new Lagrangian families which solve other objective functions with intrinsic trade-offs such as the IB Lagrangian.
	\item We show that in deterministic scenarios (and other scenarios where the IB curve shape is known) one can use the convex IB Lagrangian to obtain a desired level of performance with a single optimization. That is, there is a one-to-one mapping between the Lagrange multiplier used for the optmization and the level of compression and informativeness obtained, and we provide the exact mapping. This eliminates the need {for} multiple optimizations to select a suitable representation. 
	\item {We introduce a particular case of the convex IB Lagrangians: the shifted exponential IB Lagrangian, which allow us to approximately obtain a specific compression level in any scenario. This way, we can approximately solve the initial constrained optimization problem from Equation (\ref{eq:ib_functional}) with a single optimization}.
\end{enumerate}

Furthermore, we provide some insight for explaining why there are discontinuities in the performance levels w.r.t. the values of the Lagrange multipliers. In a classification setting, we connect those discontinuities with the intrinsic clusterization of the representations when optimizing the IB bottleneck objective.

The structure of the article is the following: In Section \ref{sec:ib_supervised_learning} we motivate the usage of the IB in supervised learning settings. Then, in Section \ref{sec:ib_deterministic} we outline the important results used about the IB curve in deterministic scenarios. Later, in Section \ref{sec:convex_ib_lagrangian} we introduce the convex IB Lagrangian and explain some of its properties like the bijective mapping between Lagrange multipliers and the compression level and the range of such multipliers. After that, we support our (proved) claims with some empirical evidence on the MNIST \citep{lecun1998gradient} {and TREC-6 \citep{li2002learning}} datasets in Section \ref{sec:experiments}. Finally, in Section \ref{sec:conclusion} we discuss our claims and empirical results. A PyTorch \citep{paszke2017automatic} implementation of the article can be found at \href{https://github.com/burklight/convex-IB-Lagrangian-PyTorch}{\texttt{https://github.com/burklight/convex-IB-Lagrangian-PyTorch}}. 

In the Appendices \ref{app:min_jce} - \ref{proof:bound_domain_beta} we provide with the proofs of the theoretical results. Then, in Appendix \ref{app:alternatives_to_convex_ib_lagrangians} we show some alternative families of Lagrangians with similar properties. Later, in Appendix \ref{app:experimental_setup_details} we provide with the precise experimental setup details to reproduce the results from the paper{, and further experimentation with different datasets and neural network architectures}. To conclude, in Appendix \ref{app:guidelines_on_choosing_proper_h} we show some guidelines on how to set the convex information bottleneck Lagrangians for practical problems.

\section{The IB in supervised learning}
\label{sec:ib_supervised_learning}

In this section{,} we will first give an overview of supervised learning in order to later motivate the usage of the information bottleneck in this setting.

\subsection{Supervised learning overview}

In supervised learning we are given a dataset $\mathcal{D}_n = \lbrace (x_i, y_i) \rbrace_{i=1}^n$ of $n$ pairs of input features and task outputs. In this case, $X$ and $Y$ are the random variables of the input features and the task outputs. We assume $x_i$ and $y_i$ are sampled i.i.d. from the true distribution $p_{(X,Y)} = p_{Y|X}p_X$. The usual aim of supervised learning is to use the dataset $\mathcal{D}_n$ to learn a particular conditional distribution $q_{\hat{Y}|X}$ of the task outputs given the input features, parametrized by $\theta$, which is a good approximation of $p_{Y|X}$. We use $\hat{Y}$ and $\hat{y}$ to indicate the predicted task output random variable and its outcome. We call a supervised learning task \textit{regression} when $Y$ is continuous-valued and \textit{classification} when it is discrete.

Usually, supervised learning methods employ intermediate representations of the inputs before making predictions about the outputs; e.g., hidden layers in neural networks (Chapter 5 from \citet{bishop2006pattern}) or transformations in a feature space through the kernel trick in kernel machines like SVMs or RVMs (Sections 7.1 and 7.2 from \citet{bishop2006pattern}). Let $T$ be a possibly stochastic function of the input features $X$ with a parametrized conditional distribution $q_{T|X}$, then, $T$ obeys the Markov condition $Y \leftrightarrow X \leftrightarrow T$. The mapping from the representation to the predicted task outputs is defined by the parametrized conditional distribution $q_{\hat{Y}|T}$. Therefore, in representation-based machine learning methods{,} the full Markov Chain is $Y \leftrightarrow X \leftrightarrow T \leftrightarrow \hat{Y}$. Hence, the overall estimation of the conditional probability $p_{Y|X}$ is given by the marginalization of the representations; i.e.,  $
	q_{\hat{Y}|X} = \mathbb{E}_{t \sim q_{T|X}} \left[ q_{\hat{Y}|T=t} \right]$\footnote{The notation $q_{\hat{Y}|T=t}$ represents the probability distribution $q_{\hat{Y}|T}(\cdot|t;\theta)$. For the rest of the text, we will use the same notation to represent conditional probability distributions where the conditioning argument is given.}.

In order to achieve the goal of having a good estimation of the conditional probability distribution $p_{Y|X}$, we usually define an \textit{instantaneous cost function} $\mathscr{j}: \mathcal{X} \times \mathcal{Y} \rightarrow \mathbb{R}$. The value of this function $\mathscr{j}(x,y;\theta)$ serves as a heuristic to measure the loss our algorithm, parametrized by $\theta$, obtains when trying to predict the realization of the task output $y$ with the input realization $x$. 

Clearly, we can be interested in minimizing the expectation of the instantaneous cost function over all the possible input features and task outputs, which we call the \textit{cost function}. However, since we only have a finite dataset $\mathcal{D}_n$ we have instead to minimize the \textit{empirical cost function}. 

\begin{Definition}[Cost function and empirical cost function]
\label{def:cost_function}
Let $X$ and $Y$ be the input features and task output random variables and $x \in \mathcal{X}$ and $y \in \mathcal{Y}$ their realizations. Let also $\mathscr{j}$ be the instantaneous cost function, $\theta$ the parametrization of our learning algorithm, and $\mathcal{D}_n = \lbrace (x_i, y_i) \rbrace_{i=1}^n$ the given dataset. Then, we define:
\begin{flalign}
  &\text{1. The cost function:} && {J}(p_{(X,Y)};\theta) = \mathbb{E}_{(x,y) \sim p_{(X,Y)}}[\mathscr{j}(x,y;\theta)] && \\
  &\text{2. The emprical cost function:} && \hat{{J}}(\mathcal{D}_n;\theta) = \frac{1}{n} \sum_{i=1}^n \mathscr{j}(x_i,y_i;\theta) &&
\end{flalign}
\end{Definition}

The discrepancy between the normal and empirical cost functions is called the \textit{generalization gap} or \textit{generalization error} (see Section 1 of \citet{xu2017information}, for instance) and {intuitively}, the smaller this gap is, the better our model generalizes; i.e., the better it will perform to new, unseen samples in terms of our cost function.

\begin{Definition}[Generalization gap]
\label{def:generalization_gap}
Let ${J}(p_{(X,Y)};\theta)$ and $\hat{{J}}(\mathcal{D}_n;\theta)$ be the cost and the empirical cost functions as defined in Definition \ref{def:cost_function}. Then, the generalization gap is defined as
\begin{equation}
	\textnormal{gen}(\mathcal{D}_n;\theta) = {J}(p_{(X,Y)};\theta) - \hat{{J}}(\mathcal{D}_n;\theta),
\end{equation}
and it represents the error incurred when the selected distribution is the one parametrized by $\theta$ when the rule $\hat{J}(\mathcal{D}_n;\theta)$ is used instead of ${J}(p_{(X,Y)};\theta)$ as the function
to minimize.
\end{Definition}

Ideally, we would want to minimize the cost function. Hence, we usually try to minimize the empirical cost function and the generalization gap simultaneously. The modifications to our learning algorithm which intend to reduce the generalization gap but not hurt the performance on the empirical cost function are known as \textit{regularization}.

\subsection{Why do we use the IB?}
\label{subsec:why_the_ib}

\begin{Definition}[Representation cross-entropy cost function] Let $X$ and $Y$ be two statistically dependent variables with joint distribution $p_{(X,Y)} = p_{Y|X}p_X$. Let also $T$ be a random variable obeying the Markov condition $Y \leftrightarrow X \leftrightarrow T$ and $q_{T|X}$ and $q_{\hat{Y}|T}$ be the encoding and decoding distributions of our model, parametrized by $\theta$. Finally, let $\mathbb{C}(p_Z || q_Z) = - \mathbb{E}_{z \sim p_Z}[\log(q_{Z}(z))]$ be the cross entropy between two probability distributions $p_Z$ and $q_Z$. Then, the cross-entropy cost function is

\begin{align}
	{J}_{\textnormal{CE}}(p_{(X,Y)};\theta) = \mathbb{E}_{(x,t) \sim q_{T|X}p_X}\left[\mathbb{C}(q_{Y|T=t}||q_{\hat{Y}|T=t}) \right] 
	= \mathbb{E}_{(x,y) \sim p_{(X,Y)}}\left[\mathscr{j}_{\textnormal{CE}}(x,y;\theta) \right],	
\end{align}

where $\mathscr{j}_{\textnormal{CE}}(x,y;\theta) = - \mathbb{E}_{t \sim q_{T|X=x}}[q_{\hat{Y}|T=t}(y|t;\theta)]$ is the instantaneous representation cross-entropy cost function and $q_{Y|T} = \mathbb{E}_{x\sim p_X}[p_{Y|X=x}q_{T|X=x}/q_{T}]$ and $q_T = \mathbb{E}_{x\sim p_X}[q_{T|X=x}]$.

\label{def:cross_entropy_cost}
\end{Definition}

The cross-entropy is a widely used cost function in classification tasks (e.g., \citet{krizhevsky2012imagenet, shore1982minimum, teahan2000text}) which has many interesting properties \citep{shore1981properties}. Moreover, it is known that minimizing the $J_{\text{CE}}(p_{(X,Y)};\theta)$ maximizes the mutual information $I(T;Y)$. That is: 

\begin{Proposition}[Minimizing the cross entropy maximizes the mutual information] Let $J_{\textnormal{CE}}(p_{(X,Y)};\theta)$ be the representation cross-entropy cost function as defined in Definition \ref{def:cross_entropy_cost}. Let also $I(T;Y)$ be the mutual information between random variables $T$ and $Y$ in the setting from Definition \ref{def:cross_entropy_cost}. Then, minimizing $J_{\textnormal{CE}}(p_{(X,Y)};\theta)$ implies maximizing $I(T;Y)$.
\label{prop:min_jce_max_ity}
\end{Proposition}

The proof of this proposition can be found in Appendix \ref{app:min_jce}.

\begin{Definition}[Nuisance]
\label{def:nuisance}
A nuisance is any random variable which affects the observed data $X$ but is not informative to the task we are trying to solve. That is, $\Xi$ is a nuisance for $Y$ if $Y \perp \Xi$ or $I(\Xi, Y) = 0$.
\end{Definition}

Similarly, we know that minimizing $I(X;T)$ minimizes the generalization gap for restricted classes when using the cross-entropy cost function (Theorem 1 of \citet{vera2018role}), and when using $I(T;Y)$ directly as an objective to maximize (Theorem 4 of \citet{shamir2010learning}). Furthermore, \citet{achille2018emergence} in Proposition 3.1 upper bound the information of the input representations, $T$, with nuisances that affect the observed data, $\Xi$, with $I(X;T)$.  Therefore, minimizing $I(X;T)$ helps generalization by not keeping useless information of $\Xi$ in our representations.

Thus, jointly maximizing $I(T;Y)$ and minimizing $I(X;T)$ is a good choice both in terms of performance in the available dataset and in new, unseen data, which motivates studies on the IB. 

\section{The Information Bottleneck in deterministic scenarios}
\label{sec:ib_deterministic}


\citet{kolchinsky2018caveats} showed that when $Y$ is a deterministic function of $X$ (i.e., $Y = f(X)$), the IB curve is piecewise linear. More precisely, it is shaped as stated in Proposition \ref{prop:ib_curve_linear}.

\begin{Proposition} [The IB curve is piecewise linear in deterministic scenarios]
\label{prop:ib_curve_linear}
 Let $X$ be a random variable and $Y = f(X)$ be a deterministic function of $X$. Let also $T$ be the bottleneck variable that solves the IB functional. Then the IB curve in the information plane is defined by the following equation:

\begin{equation}
\left\{
\begin{array}{lll}
      I(T;Y) = I(X;T) & \textnormal{ if } & I(X;T) \in [0,I(X;Y)) \\
      I(T;Y) = I(X;Y) & \textnormal{ if } & I(X;T) \geq I(X;Y)
\end{array} 
\right.
\label{eq:ib_curve_det_1}
\end{equation}

\label{th:ib_curve_piecewise_linear}

\end{Proposition}

Furthermore, they showed that the IB curve could not be explored by optimizing the IB Lagrangian for multiple $\beta$ because the curve was not strictly concave. That is, there was not a one-to-one relationship between $\beta$ and the performance level.

\begin{Theorem}[In deterministic scenarios, the IB curve cannot be explored using the IB Lagrangian] Let $X$ be a random variable and $Y = f(X)$ be a deterministic function of $X$. Let also $\Delta$ be the set of random variables $T$ obeying the Markov condition $Y \leftrightarrow X \leftrightarrow T$. Then:

\begin{enumerate}
	\item Any solution $T \in \Delta$ such that $I(X;T) \in [0,I(X;Y))$ and $I(T;Y) = I(X;T)$ solves $\argmax_{T \in \Delta} \{ \mathcal{L}_{\textnormal{IB}}(T;\beta) \}$ for $\beta = 1$. That is, many different compression and performance levels can be achieved for $\beta = 1$.
	\item Any solution $T \in \Delta$ such that $I(X;T) > I(X;Y)$ and $I(T;Y) = I(X;Y)$ solves $\argsup_{T \in \Delta} \{ \mathcal{L}_{\textnormal{IB}}(T;\beta) \}$\footnote{Note we use the supremum in this case since for $\beta = 0$ we have that $I(X;T)$ could be infinite and then the search set from Equation (\ref{eq:ib_functional}); i.e., $\lbrace T : Y \leftrightarrow X \leftrightarrow T \rbrace \cap \lbrace T : I(X;T) < \infty \rbrace $) is not compact anymore.} for $\beta = 0$. That is, many compression levels can be achieved with the same performance for $\beta = 0$.
	\item Any solution $T \in \Delta$ such that $I(X;T) = I(T;Y) = I(X;Y)$ solves $\argmax_{T \in \Delta} \{ \mathcal{L}_{\textnormal{IB}}(T;\beta) \}$ for all $\beta \in (0,1)$. That is, many different $\beta$ achieve the same compression and performance level.
\end{enumerate}

\label{th:det_ib_curve_not_explorable}
\end{Theorem}

An alternative proof for this theorem can be found in Appendix \ref{app:det_ib_curve_not_explorable}.

\section{The Convex IB Lagrangian}
\label{sec:convex_ib_lagrangian}

\subsection{Exploring the IB curve}
\label{subsec:exploring_ib_curve}

Clearly, a situation like the one depicted in Theorem \ref{th:det_ib_curve_not_explorable} is not desirable, since we cannot aim for different levels of compression or performance. For this reason, we generalize the effort from \citet{kolchinsky2018caveats} and look for families of Lagrangians which are able to explore the IB curve. Inspired by the squared IB Lagrangian, $\mathcal{L}_{\text{sq-IB}}(T;{\beta_{\text{sq}}}) = I(T;Y) - \beta_{\text{sq}} I(X;T)^2$, we look at the conditions a function of $I(X;T)$ requires in order to be able to explore the IB curve. In this way, we realize that any monotonically increasing and strictly convex function will be able to do so, and we call the family of Lagrangians with these characteristics the \textit{convex IB Lagrangians}, due to the nature of the introduced function.

\begin{Theorem}[Convex IB Lagrangians] Let $\Delta$ be the set of r.v. $T$ obeying the Markov condition $Y \leftrightarrow X \leftrightarrow T$. Then, if $u$ is a \textbf{monotonically increasing and strictly convex function}, the IB curve can always be recovered by the solutions of $\argmax_{T \in \Delta} \{\mathcal{L}_{\textnormal{IB},u}(T;{\beta_u})\}$, with 

\begin{equation}
	\mathcal{L}_{\textnormal{IB},u}(T;{\beta_u}) = I(T;Y) - \beta_u u(I(X;T)).
\end{equation}

That is, for each point $(I(X;T),I(T;Y))$ s.t. $dI(T;Y)/dI(X;T) > 0$ there is a unique $\beta_u$ for which maximizing $\mathcal{L}_{\textnormal{IB},u}(T;{\beta_u})$ achieves this solution. Furthermore, $\beta_u$ is strictly decreasing w.r.t. $I(X;T)$. We call $\mathcal{L}_{\textnormal{IB},u}(T;{\beta_u})$ the convex IB Lagrangian.

\label{th:ib_convex_lagrangians}
\end{Theorem}

The proof of this theorem can be found on Appendix \ref{proof:ib_convex_lagrangians}. Furthermore, by exploiting the IB curve duality (Lemma 10 of \citet{gilad2003information}) we were able to derive other families of Lagrangians which allow for the exploration of the IB curve (Appendix \ref{app:alternatives_to_convex_ib_lagrangians}).

\begin{Remark}
Clearly, we can see how if $u$ is the identity function (i.e., $u(I(X;T)) = I(X;T)$) then we end up with the normal IB Lagrangian. However, since the identity function is not strictly convex, it cannot ensure the exploration of the IB curve.
\end{Remark}

During the proof of this theorem we observed a relationship between the Lagrange {multipliers} and the solutions obtained of the normal IB Lagrangian $\mathcal{L}_{\textnormal{IB}}(T;\beta)$ and the convex IB Lagrangian $\mathcal{L}_{\textnormal{IB},u}(T;\beta_u)$. This relationship is formalized in the following corollary.

\begin{Corollary}[IB Lagrangian and IB convex Lagrangian connection] Let $\mathcal{L}_{\textnormal{IB}}(T;\beta)$ be the IB Lagrangian and $\mathcal{L}_{\textnormal{IB},u}(T;\beta_u)$ the convex IB Lagrangian. Then, maximizing $\mathcal{L}_{\textnormal{IB}}(T;\beta)$ and $\mathcal{L}_{\textnormal{IB},u}(T;\beta_u)$ can obtain the same point in the IB curve if $\beta_u = \beta / u'(I(X;T))$, where $u'$ is the derivative of $u$.
\end{Corollary}

This corollary {allows} us to better understand why the addition of $u$ allows for the exploration of the IB curve in deterministic scenarios. If we note that for $\beta = 1$ we can obtain any point in the increasing region of the curve, then we clearly see how evaluating $u'$ for different values of $I(X;T)$ define different values of $\beta_u$ that obtain such points. Moreover, it lets us see how if for $\beta = 0$ maximizing the IB Lagrangian could obtain any point $(I(X;Y);I(X;T))$ with $I(X;T) > I(X;Y)$, then the same happens for the IB convex Lagrangian.

\subsection{Aiming for a specific compression level}
\label{subsec:aim_for_specific_ixt}

Let $B_u$ denote the domain of Lagrange multipliers $\beta_u$ for which we can find solutions in the IB curve with the convex IB Lagrangian. Then, the convex IB Lagrangians do not only allow us to explore the IB curve with different $\beta_u$. They also allow us to identify the specific $\beta_u$ that obtains a given point $(I(X;T), I(T;Y))$, provided we know the IB curve in the information plane. Conversely, the convex IB Lagrangian allows {finding} the specific point $(I(X;T), I(T;Y))$ that is obtained by a given $\beta_u$.  

\begin{Proposition}[Bijective mapping between IB curve point and convex IB Lagrange multiplier] Let the IB curve in the information plane be known; i.e., $I(T;Y) = f_{\textnormal{IB}}(I(X;T))$ is known. Then there is a bijective mapping from Lagrange multipliers $\beta_u \in B_u \setminus \{ 0 \}$ from the convex IB Lagrangian to points in the IB curve $(I(X;T), f_{\textnormal{IB}}(I(X;T))$. Furthermore, these mappings are:

\begin{equation}
	\beta_u = \frac{df_{\textnormal{IB}}(I(X;T))}{dI(X;T)} \frac{1}{u'(I(X;T))} \quad \text{and} \quad I(X;T) = (u')^{-1} \left( \frac{df_{\textnormal{IB}}(I(X;T))}{dI(X;T)} \frac{1}{\beta_u} \right),
\end{equation}

where $u'$ is the derivative of $u$ and $(u')^{-1}$ is the inverse of $u'$.

\label{prop:bijective_mapping_beta_ixt}
\end{Proposition} 

This is especially interesting since in deterministic scenarios we know the shape of the IB curve (Theorem \ref{th:ib_curve_piecewise_linear}) and since the convex IB Lagrangians allow for the exploration of the IB curve (Theorem \ref{th:ib_convex_lagrangians}). A proof for Proposition \ref{prop:bijective_mapping_beta_ixt} can be found in Appendix \ref{proof:bijective_mapping}.

\begin{Remark}
Note that the definition from \citet{tishby2000information} $\beta = df_{\textnormal{IB}}(I(X;T))/dI(X;T)$ only allows for a bijection between $\beta$ and $I(X;T)$ if $f_{\textnormal{IB}}$ is a strictly convex, and known function, and we have seen this is not the case in deterministic scenarios (Theorem \ref{th:det_ib_curve_not_explorable}).
\end{Remark}

A direct result derived from this proposition is that we know the domain of Lagrange multipliers, $B_u$, which allow for the exploration of the IB curve if the shape of the IB curve is known. Furthermore, if the shape is not known we can at least bound that range. 

\begin{Corollary}[Domain of convex IB Lagrange multiplier with known IB curve shape] Let the IB curve in the information plane be $I(T;Y) = f_{\textnormal{IB}}(I(X;T))$ and let $I_{\textnormal{max}} = I(X;Y)$. Let also $I(X;T) = r_{\textnormal{max}}$ be the minimum mutual information s.t. $f_{\textnormal{IB}}(r_{\textnormal{max}}) = I_{\textnormal{max}}$; i.e., {$r_{\textnormal{max}} = \arginf_r \{ f_{\textnormal{IB}}(r)\} \textnormal{ s.t. }  f_{\textnormal{IB}}(r) = I_{\textnormal{max}}$} \footnote{{Note that there are some scenarios where $r_{\textnormal{max}} \rightarrow \infty$ (see, e.g., \citep{du2017strong}). In these scenarios $\beta_{u,\textnormal{min}} = \lim_{r \rightarrow \infty} \left \{ f'_{\textnormal{IB}}(r) / u'(r) \right \} \geq 0$.} }. Then, the range of Lagrange multipliers that allow the exploration of the IB curve with the convex IB Lagrangian is $B_u = [\beta_{u,\textnormal{min}},\beta_{u,\textnormal{max}}]$, with 

\begin{equation}
	\beta_{u,\textnormal{min}} = \lim_{r \rightarrow r_{\textnormal{max}}^{-}} \left \{ \frac{f'_{\textnormal{IB}}(r)}{u'(r)}\right \} \quad \text{and} \quad \beta_{u,\textnormal{max}} = \lim_{r \rightarrow 0^+} \left \{ \frac{f'_{\textnormal{IB}}(r)}{u'(r)}\right \},
\end{equation}

where $f'_{\textnormal{IB}}(r)$ and $u'(r)$ are the derivatives of $f_{\textnormal{IB}}(I(X;T))$ and $u(I(X;T))$ w.r.t. $I(X;T)$ evaluated at $r$ respectively.

\label{cor:domain_conv_ib_lagrange}
\end{Corollary}

\begin{Corollary}[Domain of convex IB Lagrange multiplier bound] 
\label{cor:bound_domain}
The range of the Lagrange multipliers that allow the exploration of the IB curve is contained by $[0,\beta_{\textnormal{u,top}}]$ which is also contained by $[0,\beta_{u,\textnormal{top}}^{+}]$, where

\begin{equation}
	\beta_{u,\textnormal{top}} =  \frac{(\inf_{\Omega_x \subset \mathcal{X}} \lbrace \beta_0(\Omega_x)  \rbrace)^{-1}}{\lim_{r \rightarrow 0^+} \left \{ u'(r)\right \}}, \text{ and } \beta_{u,\textnormal{top}}^{+} =  \frac{1}{\lim_{r \rightarrow 0^+} \left \{ u'(r)\right \}},
\end{equation}

$u'(r)$ is the derivative of $u(I(X;T))$ w.r.t. $I(X;T)$ evaluated at $r$, $\mathcal{X}$ is the set of possible realizations of $X$ and $\beta_0$\footnote{Note in \citep{wu2019learnability} they consider the dual problem (see Appendix \ref{app:alternatives_to_convex_ib_lagrangians}) so when they refer to $\beta^{-1}$ it translates to $\beta$ in this article.} and $\Omega_x$ are defined as in \citep{wu2019learnability}. That is, 
$B_u \subseteq [0,\beta_{\textnormal{u,top}}] \subseteq [0,\beta_{\textnormal{u,top}}^+]$. 

\end{Corollary}

Corollaries \ref{cor:domain_conv_ib_lagrange} and \ref{cor:bound_domain} allow us to reduce the range search for $\beta$ when we want to explore the IB curve. Practically, $\inf_{\Omega_x \subset \mathcal{X}} \lbrace \beta_0(\Omega_x)  \rbrace$ might be difficult to calculate so  \citet{wu2019learnability} derived an algorithm to approximate it. However, we still recommend setting the numerator to 1 for simplicity. The proofs for both corollaries are found in Appendices \ref{proof:domain_beta_f_known} and \ref{proof:bound_domain_beta}.

\section{Experimental support}
\label{sec:experiments}

In order to showcase our claims we use the MNIST \citep{lecun1998gradient} {and the TREC-6 \citep{li2002learning} datasets} . We modify the nonlinear-IB method \citep{kolchinsky2017nonlinear}, which is a neural network that minimizes the cross-entropy while also minimizing a differentiable kernel-based estimate of $I(X;T)$ \citep{kolchinsky2017estimating}. Then, we use this technique to maximize a lower bound on the convex IB Lagrangians by applying the functions $u$ to the $I(X;T)$ estimate.

\begin{figure}
\centering
\includegraphics[width=\linewidth]{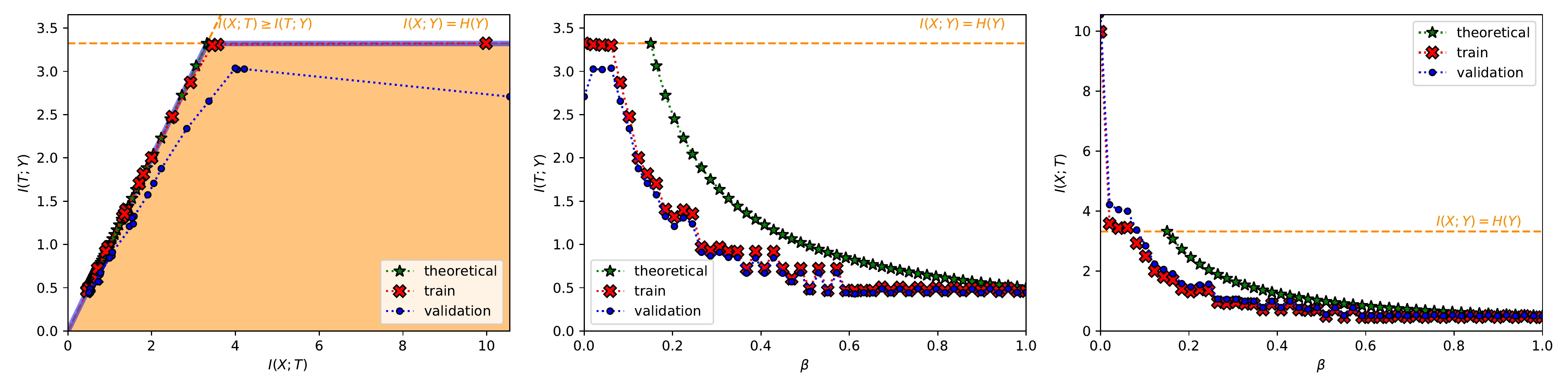}
\includegraphics[width=\linewidth]{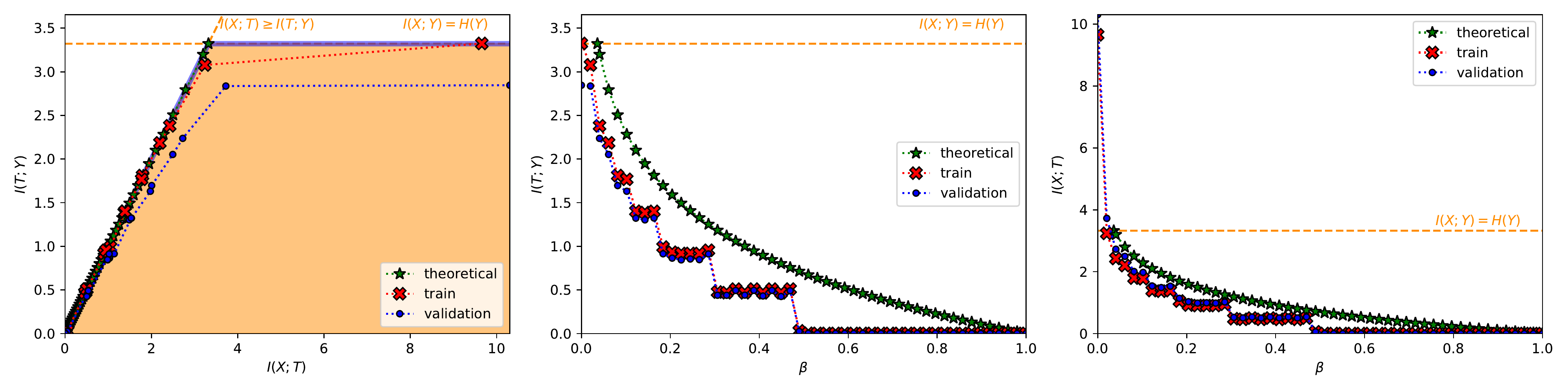}
\caption{The top row shows the results for the power IB Lagrangian  with $\alpha = 1$, and the bottom row for the exponential IB Lagrangian with $\eta = 1$, {both in the MNIST dataset}. In each row, from left to right it is shown (i) the information plane, where the region of possible solutions of the IB problem is shadowed in light orange and the information-theoretic limits are the dashed orange line; (ii) $I(T;Y)$ as a function of $\beta_u$; and (iii) the compression $I(X;T)$ as a function of $\beta_u$. In all plots{,} the red crosses joined by a dotted line represent the values computed with the training set, the blue dots the values computed with the validation set and the green stars the theoretical values computed as dictated by Proposition \ref{prop:bijective_mapping_beta_ixt}. Moreover, in all plots{,} it is indicated $I(X;Y) = H(Y) = \log_2(10)$ in a dashed, orange line. All values are shown in bits.} 
\label{fig:example_performance}
\end{figure}

\begin{figure}
\centering
\begin{subfigure}[b]{0.4\textwidth}
\includegraphics[width=\textwidth]{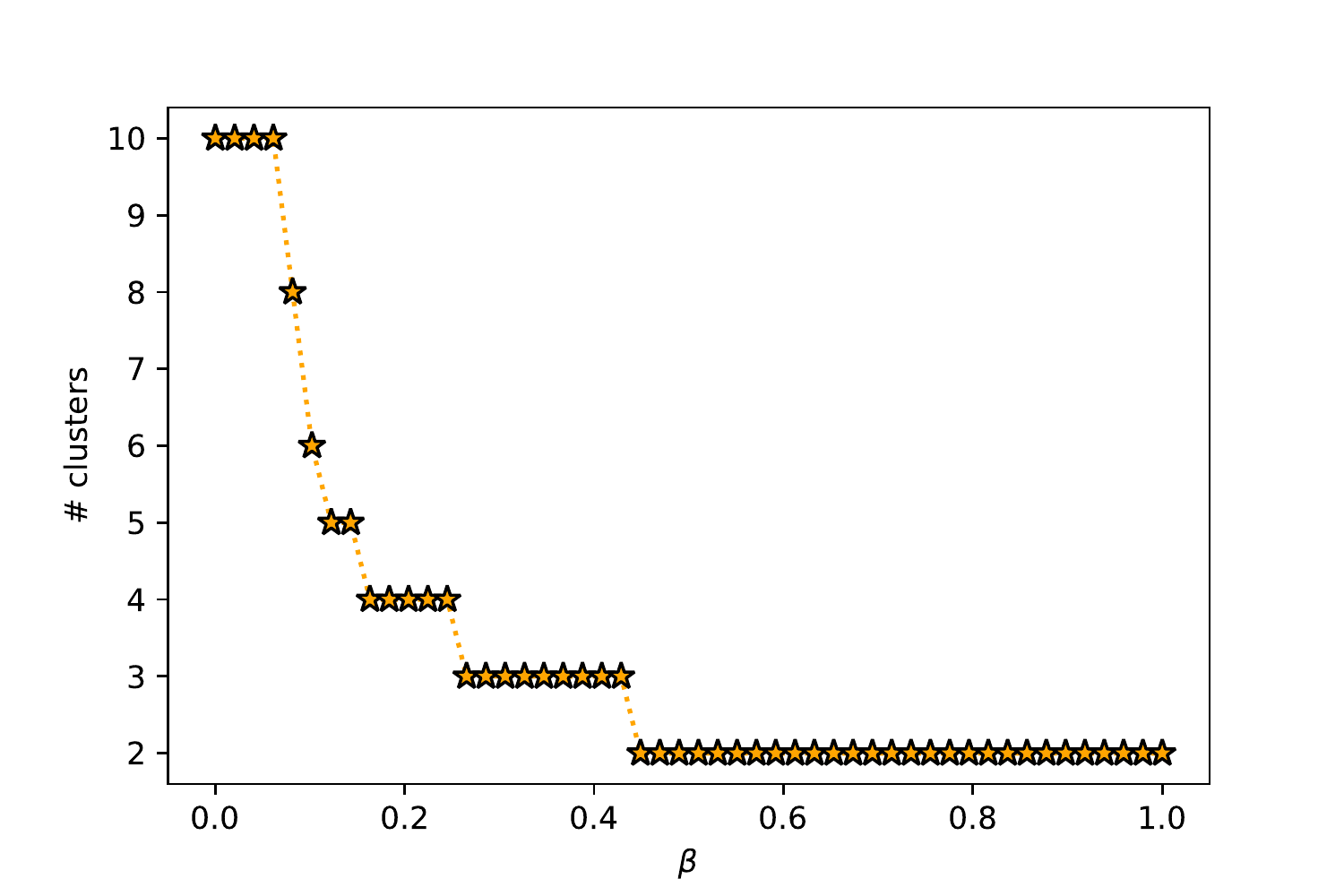}
\caption{Number of clusters for different $\beta_{\textnormal{pow}}$.}
\label{fig:example_number_clusters_power_alpha_1}
\end{subfigure}
\begin{subfigure}[b]{0.55\textwidth}
\includegraphics[width=\textwidth]{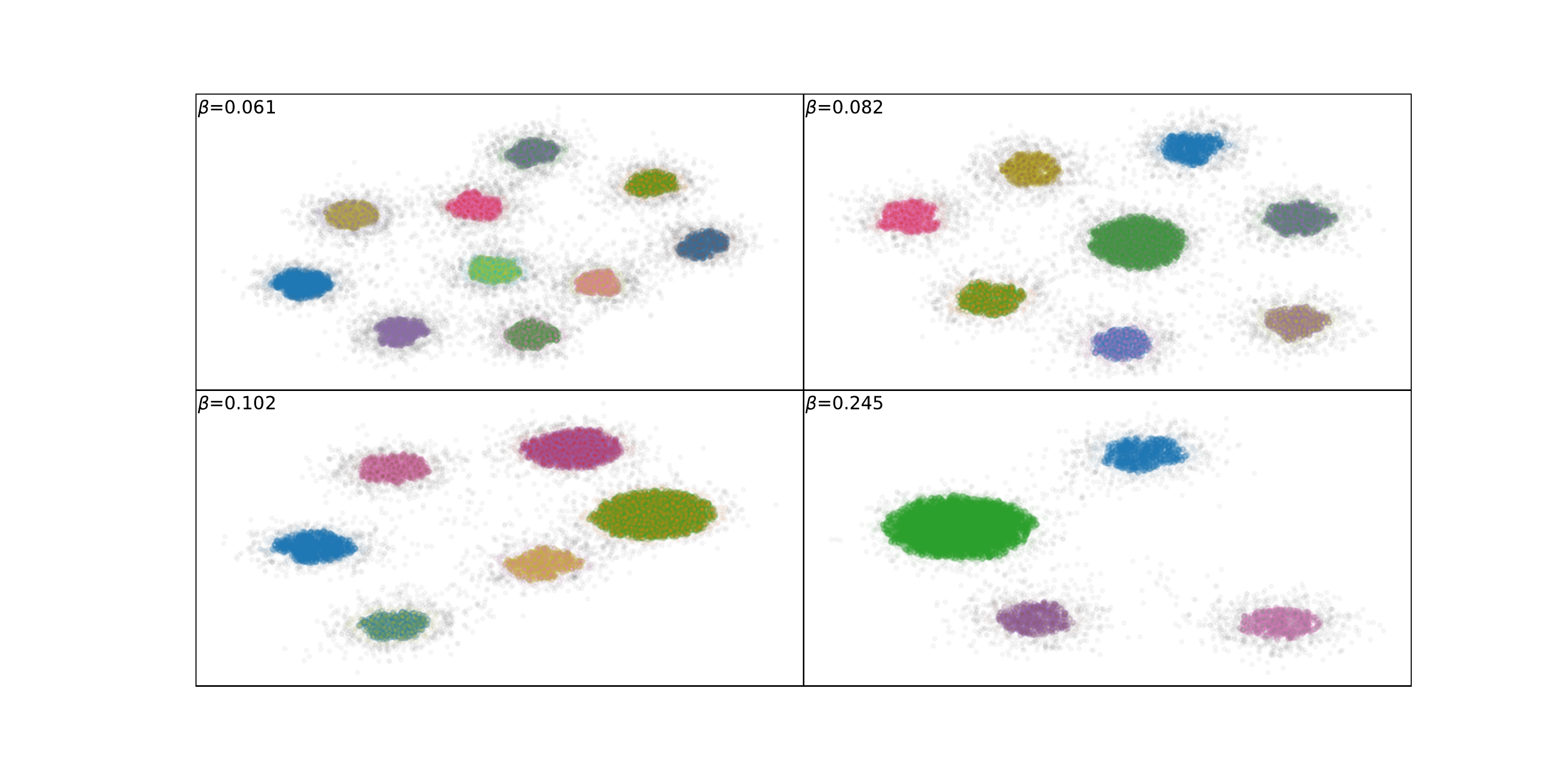}
\caption{Example of clusters for different $\beta_{\textnormal{pow}}$.}
\label{fig:example_clusters}
\end{subfigure}
\caption[Depiction of the clusterization behavior of the bottleneck variable for the power IB Lagrangian with $\alpha = 1$.]{Depiction of the clusterization behavior\footnotemark of the bottleneck variable for the power IB Lagrangian {in the MNIST dataset} with $\alpha = 1$. }
\label{fig:clusters_alpha_1}
\end{figure}
\footnotetext{The clusters were obtained using the DBSCAN algorithm \citep{ester1996density, schubert2017dbscan}.}

The network structure is the following: First, a stochastic encoder\footnote{The encoder needs to be stochastic to (i) ensure a finite and well-defined mutual information \citep{kolchinsky2018caveats,amjad2019learning} and (ii) make gradient-based optimization methods over the IB Lagrangian useful \citep{amjad2019learning}.} $T = f_{\textnormal{enc}}(X;\theta) + W$ with $p_W = \mathcal{N}(0,I_d)$ such that $T \in \mathbb{R}^d$, where $d$ is the dimension of the bottleneck variable. Second, a deterministic decoder $q_{\hat{Y}|T} = f_{\textnormal{dec}}(T;\theta)$. For the MNIST dataset both the encoder and the decoder are fully-connected networks, for a fair comparison with \citep{kolchinsky2017nonlinear}. {For the TREC-6 dataset, the encoder is a set of convolutions of word embeddings followed by a fully-connected network and the decoder is also a fully-connected network.} For further details about the experiment setup, additional results for different values of $\alpha$ and $\eta$ {and supplementary experimental results for different datasets and network architectures,} please refer to Appendix \ref{app:experimental_setup_details}.


In Figure \ref{fig:example_performance} we show our results for two particularizations of the convex IB Lagrangians:  
\begin{enumerate}
\item the \textbf{power IB Lagrangians}\footnote{Note when $\alpha = 1$ we have the squared IB functional from \citet{kolchinsky2018caveats}.}: $\mathcal{L}_{\textnormal{IB,pow}}(T;{\beta_\textnormal{pow}},\alpha) = I(T;Y) - \beta_{\textnormal{pow}} I(X;T)^{(1+\alpha)}$, $\alpha > 0$ .
\item the \textbf{exponential IB Lagrangians}: $\mathcal{L}_{\textnormal{IB,exp}}(T;{\beta_\textnormal{exp}},\eta) = I(T;Y) - \beta_{\textnormal{exp}} \exp(\eta I(X;T))$, $\eta > 0$.
\end{enumerate}

We can clearly see how both Lagrangians are able to explore the IB curve (first column from Figure \ref{fig:example_performance}) and how the theoretical performance trend of the Lagrangians matches the experimental results (second and third columns from Figure \ref{fig:example_performance}). There are small mismatches between the theoretical and experimental performance. This is because using the nonlinear-IB, as stated by \citet{kolchinsky2018caveats}, does not guarantee that we find optimal representations due to factors like: (i) {inaccurate} estimation of $I(X;T)$, (ii) restrictions on the structure of $T$, (iii) use of an estimation of the decoder instead of the real one and (iv) the typical non-convex optimization issues that arise with gradient-based methods. The main difference comes from the discontinuities in performance for increasing $\beta$, which cause is still unknown (cf. \citet{wu2019learnability}). It has been observed, however, that the bottleneck variable performs an intrinsic clusterization in classification tasks  (see, for instance \citep{kolchinsky2017nonlinear, kolchinsky2018caveats, alemi2018uncertainty} or Figure \ref{fig:example_clusters}). We observed how this clusterization matches with the quantized performance levels observed (e.g., compare Figure \ref{fig:example_number_clusters_power_alpha_1} with the top center graph in Figure \ref{fig:example_performance}); with maximum performance when the number of clusters is equal to the cardinality of $Y$ and reducing performance with a reduction of the number of clusters, which is in line with the concurrent work from \citet{wu2020phase}. We do not have a mathematical proof for the exact relationship between these two phenomena; however, we agree with \citet{wu2019learnability} that it is an interesting matter and hope this observation serves as motivation to derive new theory.

\begin{figure}
\centering
\includegraphics[width=0.6\textwidth]{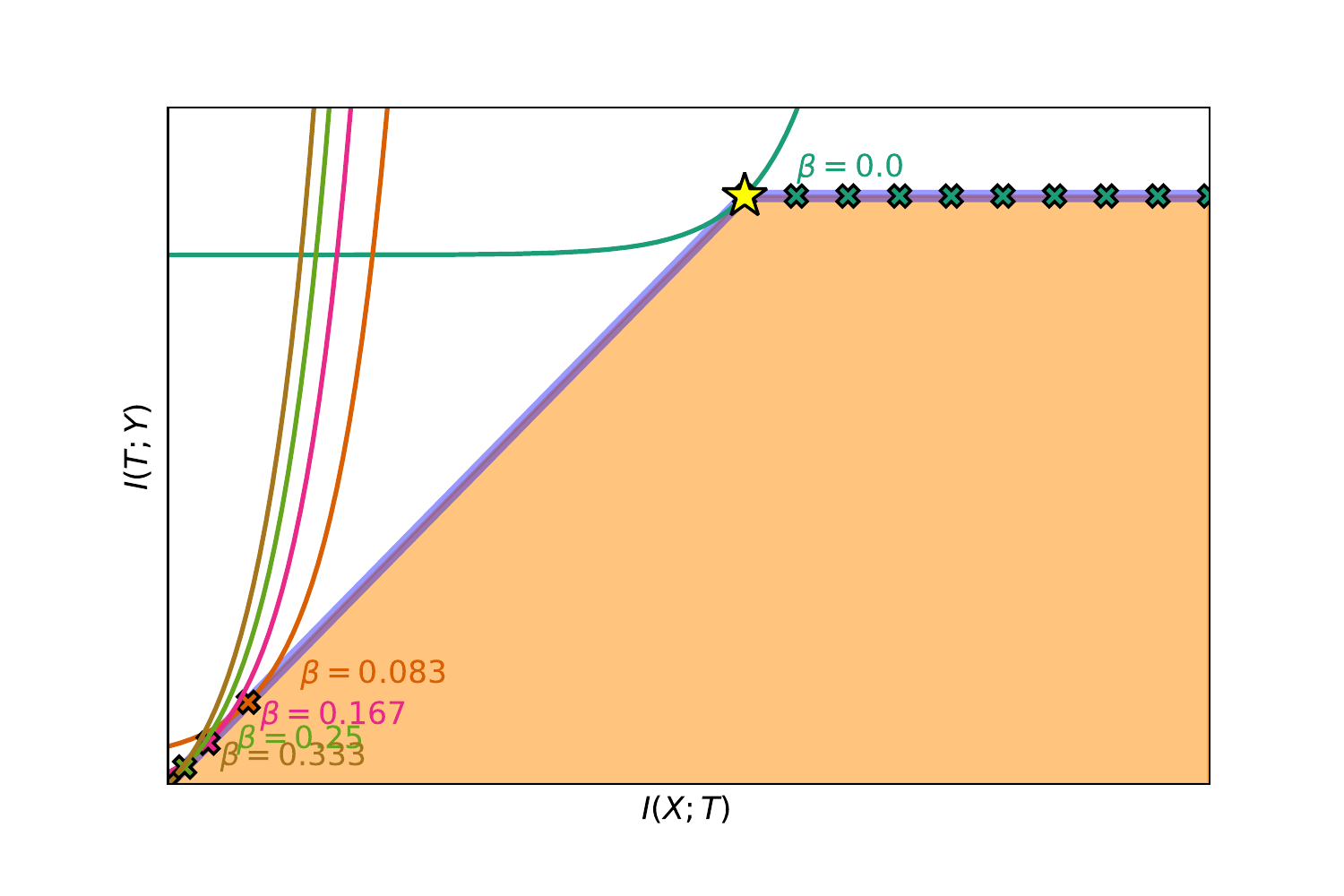}
\caption[Example of value convergence with the exponential IB Lagrangian with $\eta = 3$. We show the intersection of the isolines of $\mathcal{L}_{\textnormal{IB,exp}}(T;{\beta_{\textnormal{exp}}})$ for different $\beta_{\textnormal{exp}} \in B_{\textnormal{exp}}$ with the IB curve.]{Example of value convergence with the exponential IB Lagrangian with $\eta = 3$. We show the intersection of the isolines of $\mathcal{L}_{\textnormal{IB,exp}}(T;{\beta_{\textnormal{exp}}})$ for different $\beta_{\textnormal{exp}} \in B_{\textnormal{exp}} \approx [1.56\cdot10^{-5}, 3^{-1}]$ using Corollary \ref{cor:domain_conv_ib_lagrange}.}
\label{fig:value_convergence}
\end{figure}

\begin{figure}[t]
\centering
\includegraphics[width=\linewidth]{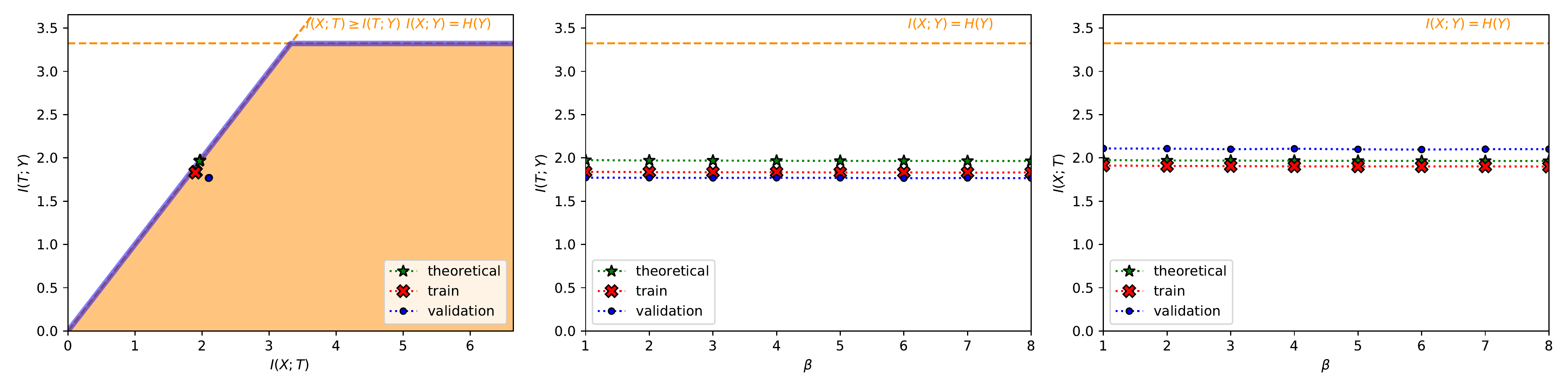}
\includegraphics[width=\linewidth]{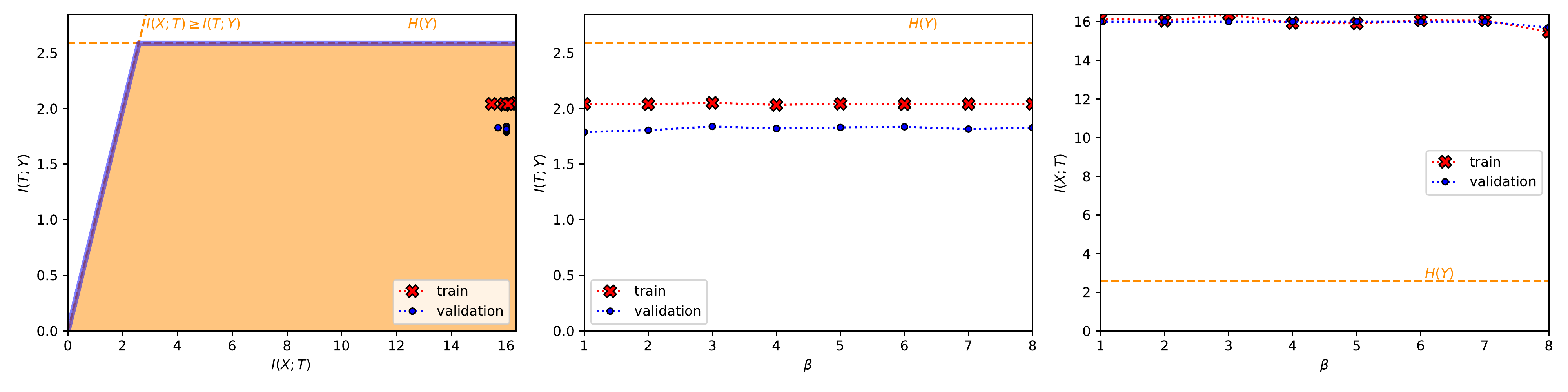}
\caption{{Example of value convergence exploitation with the shifted exponential Lagrangian with $\eta = 200$. In the top row, for the MNIST dataset aiming for a compression level $r^* = 2$ and in the bottom row, for the TREC-6 dataset aiming for a compression level of $r^* = 16$. In each row, from left to right it is shown (i) the information plane, where the region of possible solutions of the IB problem is shadowed in light orange and the information-theoretic limits are the dashed orange line; (ii) $I(T;Y)$ as a function of $\beta_u$; and (iii) the compression $I(X;T)$ as a function of $\beta_u$. In all plots, the red crosses joined by a dotted line represent the values computed with the training set, the blue dots the values computed with the validation set and the green stars the theoretical values computed as dictated by Proposition \ref{prop:bijective_mapping_beta_ixt}. Moreover, in all plots, it is indicated $H(Y)$ in a dashed, orange line. All values are shown in bits.} }
\label{fig:example_performance_value_convergence}
\end{figure}

{In practice, there are different criteria for choosing the function $u$. For instance, the exponential IB Lagrangian could be more desirable than the power IB Lagrangian when we want to draw the IB curve since it has a finite range of $\beta_u$. This is $B_u = [(\eta \exp(\eta I_{\textnormal{max}}))^{-1}, \eta^{-1}]$ for the exponential IB Lagrangian vs. $B_u = [((1+\alpha)I_{\textnormal{max}}^\alpha)^{-1}, \infty)$ for the power IB Lagrangian. Furthermore, there is a trade-off between (i) how much the selected $u$ function {resembles} a linear function in our region of interest; e.g., with $\alpha$ or $\eta$ close to zero, since it will suffer from similar problems as the original IB Lagrangian; and (ii) how fast it grows in our region of interest; e.g., higher values of $\alpha$ or $\eta$, since it will suffer from value convergence; i.e., optimizing for separate values of $\beta_u$ will achieve similar levels of performance (Figure \ref{fig:value_convergence}). Please, refer to Appendix \ref{app:guidelines_on_choosing_proper_h} for a more thorough explanation of {these two phenomena}.

{Particularly, the value convergence phenomenon can be exploited in order to approximately obtain a particular level of compression $r^*$, both for known and unkown IB curves (see Appendix \ref{app:guidelines_on_choosing_proper_h} or the example in Figure \ref{fig:example_performance_value_convergence}). For known IB curves, we also know the achieved predictability $I(T;Y)$ since it is the same as the level of compression $I(X;T)$. For this exploitation, we can employ the shifted version of the exponential IB Lagrangian (which is also a particular case of the convex IB Lagrangian):
\begin{itemize}
\item the \textbf{shifted exponential IB Lagrangians}: $$\mathcal{L}_{\textnormal{IB,sh-exp}}(T;{\beta_\textnormal{sh-exp}},\eta,r^*) = I(T;Y) - \beta_{\textnormal{sh-exp}} \exp(\eta (I(X;T)-r^*)) \textnormal{ , } \eta > 0 \textnormal{ , } r^* \in [0,\infty).$$
\end{itemize}
For this Lagrangian, the optimization procedure converges to representations with approximately the desired compression level $r^*$ if the hyperparameter $\eta$ is set to a large value.}

{In Figure \ref{fig:example_performance_value_convergence} we show the results of aiming for a compression level of $r^* = 2$ bits in the MNIST dataset and of $r^* =16$ bits in the TREC-6 dataset, both with $\eta = 200$. We can see how for different values of $\beta_{\textnormal{sh-exp}}$ we can obtain the same desired compression level, which makes this method stable to variations in the Lagrange multiplier selection.}

To sum up, in order to achieve a desired level of performance with the convex IB Lagrangian as an objective one should:

\begin{enumerate}
\item In a deterministic or close to {a} deterministic setting (see $\epsilon$-deterministic definition in \citet{kolchinsky2018caveats}): Use the adequate $\beta_u$ for that performance using Proposition \ref{prop:bijective_mapping_beta_ixt}. Then if the {performance} is lower than desired, i.e., we are placed in the wrong performance plateau, gradually reduce the value of $\beta_u$ until reaching the previous performance plateau. {Alternatively, exploit the value convergence phenomenon with, for instance, the shifted exponential IB Lagrangian}.
\item In a stochastic setting: {Exploit the value convergence phenomenon with, for instance, the shifted exponential IB Lagrangian. Alternatively,} draw the IB curve with multiple values of $\beta_u$ on the range defined by Corollary \ref{cor:bound_domain} and select the representations that best fit their interests.
\end{enumerate}

\section{Conclusion}
\label{sec:conclusion}

The information bottleneck is a widely used and studied technique. However, it is known that the IB Lagrangian cannot be used to achieve varying levels of performance in deterministic scenarios. Moreover, in order to achieve a particular level of performance multiple optimizations with different Lagrange multipliers must be done to draw the IB curve and select the best traded-off representation. 

In this article we introduced a general family of Lagrangians which allow to (i) achieve varying levels of performance in any scenario, and (ii) pinpoint a specific Lagrange multiplier $\beta_u$ to optimize for a specific performance level in known IB curve scenarios; e.g., deterministic. Furthermore, we showed the $\beta_u$ domain when the IB curve is known and a $\beta_u$ domain bound for exploring the IB curve when it is {unknown}. This way we can reduce and/or avoid multiple optimizations and, hence, reduce the computational effort for finding well traded-off representations. {Moreover, (iii) when the IB curve is not known, we saw how we can exploit the value convergence issue of the convex IB Lagrangian to approximately obtain a specific compression level for both known and unknown IB curve shapes.} Finally, (iv) we provided some insight to the discontinuities on the performance levels w.r.t. the {Lagrange} multipliers by connecting those with the intrinsic clusterization of the bottleneck variable.



\bibliography{references}

\appendix
\section{Proof of Proposition \ref{prop:min_jce_max_ity}}
\label{app:min_jce}

\begin{proof}
We can easily prove this statement by finding $I(T;Y)$ is lower bounded by the $\gamma J_{\textnormal{CE}}(p_{(X,Y)};\theta) + C$ where $\gamma < 0$ and $C$ does not depend on $T$. This way maximizing such lower bound would be equivalent to minimizing $J_{\textnormal{CE}}(p_{(X,Y)};\theta)$ and, moreover, it would imply maximizing $I(T;Y)$.

We can find such an expression as follows:

\begin{align}
	I(T;Y) &= \mathbb{E}_{(y,t) \sim q_{Y|T}q_T} \left[ \log \left(\frac{q_{Y|T=t}(y|t;\theta)}{p_Y(y)}\right) \right] = H(Y) + \mathbb{E}_{(y,t) \sim q_{Y|T}q_T} \left[ \log(q_{Y|T=t}(y|t;\theta)) \right] \label{eq:prop1_def}\\
	&= H(Y) + \mathbb{E}_{t \sim q_T} \left[ D_{\textnormal{KL}}\left(q_{Y|T=t} || q_{\hat{Y}|T=t}\right) \right] + \mathbb{E}_{(y,t) \sim q_{Y|T}q_T} \left[ \log(q_{\hat{Y}|T}(y|t;\theta)) \right] \label{eq:prop1_mult_div}\\
	&\geq H(Y) +  \mathbb{E}_{(x,y,t) \sim q_{Y|T}q_{T|X}p_X} \left[ \log(q_{\hat{Y}|T=t}(y|t,\theta)) \right] = H(Y) - \mathbb{E}_{(x,t) \sim q_{T|X}p_X} \left[ \mathbb{C}(q_{Y|T=t}||q_{\hat{Y}|T=t}) \right] \label{eq:prop1_kldiv_pos}\\
	&= H(Y) - J_{\textnormal{CE}}(p_{(X,Y)};\theta).
\end{align}

Here, in Equation (\ref{eq:prop1_def}) we just used the definition of the mutual information between two random variables, and then we decoupled it using the definition of the entropy of a variable\footnote{Note we used $H(\cdot)$ which is usually employed for discrete variables. However, in this setting $H(\cdot)$ could also refer to the differential entropy $h(\cdot)$ of a continuous random variable, since we employed the general definition using the expectation.}. Then, in Equation (\ref{eq:prop1_mult_div}) we only multiplied and divided by $q_{\hat{Y}|T}$ inside the logarithm and employed the definition of the Kullback-Leibler divergence. Finally, in Equation (\ref{eq:prop1_kldiv_pos}) we first used the fact the Kullback-Leibler divergence is always positive (Theorem 2.6.3 from \citet{cover2012elements}) and then the properties of the Markov Chain $T \leftrightarrow X \leftrightarrow Y$.

Therefore, since $H(Y)$ does not depend on $T$ and we have a negative multiplicative term on $J_{\textnormal{CE}}(p_{(X,Y)};\theta)$ the proposition is proved.
\end{proof}

\section{Alternative proof of Theorem \ref{th:det_ib_curve_not_explorable}}
\label{app:det_ib_curve_not_explorable}

\begin{proof}
We will proof all the enumerated statements sequentially, since the third one requires from the two first ones to be proved. 

\begin{enumerate}
	\item Proposition \ref{prop:ib_curve_linear} states that the IB curve in the information plane follows the equation $I(T;Y) = I(X;T)$ if $I(X;T) \in [0,I(X;Y))$. Then, since $\beta = dI(T;Y)/dI(X;T)$ \citep{tishby2000information}, we know $\beta = 1$ in all these points. Therefore, for $\beta = 1$ all points ($I(X;T),I(X;T)$) such that $I(X;T) \in [0,I(X;Y))$ are solutions of optimizing the IB Lagrangian.
	\item Similarly, Proposition \ref{prop:ib_curve_linear} states that the IB curve follows the equation $I(T;Y) = I(X;Y)$ if $I(X;T) \geq I(X;Y)$. Then, since $\beta = dI(T;Y)/dI(X;T)$ \citep{tishby2000information}, we know $\beta = 0$ in all points such that $I(X;T) > I(X;Y)$. We cannot ensure it at $I(X;T) = I(X;Y)$ since $\beta = 1$ for $I(X;T) = \lim_{\epsilon \rightarrow 0^+} \lbrace I(X;Y) - \epsilon \rbrace$.
	\item Finally, in order to prove the last statement we will first prove that if $\beta \in (0,1)$ achieves a solution, it is $(I(X;Y),I(X;Y))$. Then, we will prove that if the solution $(I(X;Y),I(X;Y))$ exists, this can be yield by any $\beta \in (0,1)$. Hence, the solution $(I(X;Y),I(X;Y))$ is achieved $\forall \beta \in (0,1)$ and it is the only solution achievable.
	
	\begin{enumerate}
	\item Since the IB curve is concave we know $\beta$ is non-increasing in $I(X;T) \in \mathbb{R}^+$. We also know $\beta = 1$ at the points in the IB curve where $I(X;T) \leq \lim_{\epsilon \rightarrow 0^+} \lbrace I(X;Y) - \epsilon \rbrace$ and $\beta = 1$ at the points in the IB curve where $I(X;T) \geq \lim_{\epsilon \rightarrow 0^+} \lbrace I(X;Y) + \epsilon \rbrace$. Hence, if we achieve a solution with $\beta \in (0,1)$, this solution is $I(X;T) = I(T;Y) = I(X;Y)$.
	\item We can upper bound the IB Lagrangian by
	
	\begin{equation}
	\mathcal{L}_{\textnormal{IB}}(T;\beta) = I(T;Y) - \beta I(X;T) \leq (1 - \beta) I(T;Y)  \leq (1 - \beta) I(X;Y),
	\label{eq:upper_bound_lagr_for_th}
	\end{equation}
	
	where the first and second inequalities use the DPI (Theorem 2.8.1 from \citet{cover2012elements}).
	
	Then, we can consider the point of the IB curve $(I(X;Y),I(X;Y))$. Since the function is concave a tangent line to $(I(X;Y),I(X;Y))$ exists such that all other points in the curve lie below this line. Let $\beta$ be the slope of this curve (which we know it is from \citet{tishby2000information}). Then,
	
	\begin{equation}
	I(X;Y) - \beta I(X;Y) = (1 - \beta) I(X;Y) \geq F_{\textnormal{IB,max}}(r) - \beta r, \ \forall r \in [0,\infty).
	\end{equation}
	
	As we see, by the upper bound on the IB Lagrangian from Equation (\ref{eq:upper_bound_lagr_for_th}), if the point $(I(X;Y),I(X;Y))$ exists, any $\beta$ can be the slope of the tangent line to $(I(X;Y),I(X;Y))$ that ensures concavity.

	\end{enumerate}
\end{enumerate}
\end{proof}

\section{Proof of Theorem \ref{th:ib_convex_lagrangians}}
\label{proof:ib_convex_lagrangians}

\begin{proof}
	We start the proof by remembering the optimization problem at hand (Definition \ref{def:ib_functional}):
	
	 \begin{equation}
	 	F_{\textnormal{IB,max}}(r) = \max_{T \in \Delta} \{I(T;Y)\} \textnormal{ s.t. } I(X;T) \leq r
	 \end{equation}
	 
	 We can modify the optimization problem by 
	 
	 	 \begin{align}
	 \max_{T \in \Delta} \{I(T;Y)\} \textnormal{ s.t. } u(I(X;T)) \leq u(r) 
	 \end{align}
	 	 
	 \textbf{iff $u$ is a monotonically non-decreasing function} since otherwise $u(I(X;T)) \leq u(r)$ would not hold necessarily. Now, let us assume $\exists T^* \in \Delta$ and $\beta_u^*$ s.t. $T^*$ maximizes $\mathcal{L}_{\textnormal{IB},u}(T;{\beta_u^*})$ over all $T \in \Delta$, and  $I(X;T^*) \leq r$. Then, we can operate as follows:
	 
	 	 \begin{align}
	 	\max_{\substack{T \in \Delta \\ u(I(X;T)) \leq u(r)}} \{I(T;Y)\} &= \max_{\substack{T \in \Delta \\ u(I(X;T)) \leq u(r)}} \{I(T;Y) - \beta_u^* (u(I(X;T)) - u(r) + \xi) \} \label{eq:max_ib_func_1}\\
	 	&\leq  \max_{T \in \Delta} \{I(T;Y) - \beta_u^* (u(I(X;T)) - u(r) + \xi) \} \label{eq:max_ib_func_2}\\
	 	&=  I(T^*;Y) - \beta_u^* (u(I(X;T^*) - u(r) + \xi) = I(T^*;Y). \label{eq:max_ib_func_4}
	 \end{align}
	
	 Here, the equality from Equation (\ref{eq:max_ib_func_1}) comes from the fact that since $I(X;T) \leq r$, then $\exists \xi \geq 0$ s.t. $u(I(X;T)) - u(r) + \xi = 0$. Then, the inequality from Equation (\ref{eq:max_ib_func_2}) holds since we have expanded the optimization search space. Finally, in Equation (\ref{eq:max_ib_func_4}) we use that $T^*$ maximizes $\mathcal{L}_{\textnormal{IB},u}(T;{\beta_u^*})$ and that $I(X;T^*) \leq r$. 
	 
	 Now, we can exploit that $u(r)$ and $\xi$ do not depend on $T$ and drop them in the maximization in Equation (\ref{eq:max_ib_func_2}). We can then realize we are maximizing over $\mathcal{L}_{\textnormal{IB},u}(T;{\beta_u^*})$; i.e., 
	 
	 \begin{align}
	 \max_{\substack{T \in \Delta \\ u(I(X;T)) \leq u(r)}} \{I(T;Y)\} &\leq \max_{T \in \Delta} \{I(T;Y) - \beta_u^* (u(I(X;T)) - u(r) + \xi) \} \\
	 &= \max_{T \in \Delta} \{I(T;Y) - \beta_u^* (I(X;T)) \} = \max_{T \in \Delta} \{\mathcal{L}_{\textnormal{IB},u}(T;{\beta_u^*}) \} . \label{eq:max_ib_func_3}
	 \end{align}
	 	 
	 Therefore, since $I(T^*;Y)$ satisfies both the maximization with $T^* \in \Delta$ and the constraint $I(X;T^*) \leq r$, maximizing $\mathcal{L}_{\textnormal{IB},u}(T;{\beta_u^*})$ obtains $F_{\textnormal{IB,max}}(r)$.
	 
	 Now, we know if such $\beta_u^*$ exists, then the solution of the Lagrangian will be a solution for $F_{\textnormal{IB,max}}(r)$. Then, if we consider Theorem 6 from the Appendix of \citet{courcoubetis2003pricing} and consider the maximization problem instead of the minimization problem, we know if both $I(T;Y)$ and $-u(I(X;T))$ are concave functions, then a set of Lagrange multipliers $S_u^*$ exists with these conditions. We can make this consideration because $f$ is concave if $-f$ is convex and $\max \{f\} = \min \{-f\}$. We know $I(T;Y)$ is a concave function of $T$ for $T \in \Delta$ (Lemma 5 of \citet{gilad2003information}) and $I(X;T)$ is convex w.r.t. $T$ given $p_X$ is fixed (Theorem 2.7.4 of \citet{cover2012elements}). Thus, if we want $-u(I(X;T))$ to be concave \textbf{we need $u$ to be a convex function}.
	 
	 Finally, we will look at the conditions of $u$ so that for every point $(I(X;T), I(T;Y))$ in the IB curve, there exists a unique $\beta_u^*$ s.t. $\mathcal{L}_{\textnormal{IB},\textnormal{u}}(T;{\beta_u^*})$ is maximized. That is, the conditions of $u$ s.t. $|S_u^*|=1$. For this purpose we will look at the solutions of the Lagrangian optimization:
	 
	 \begin{equation}
	 \frac{d\mathcal{L}_{\textnormal{IB},u}(T;{\beta_u})}{dT} = \frac{d(I(T;Y) - \beta_u u(I(X;T)))}{dT} = \frac{dI(T;Y)}{dT} - \beta_u \frac{du(I(X;T))}{dI(X;T)} \frac{dI(X;T)}{dT} = 0
	 \label{eq:sol_lagrangian_h}
	 \end{equation}
	 
	 Now, if we integrate both sides of Equation (\ref{eq:sol_lagrangian_h}) over all $T \in \Delta$ we obtain 
	 
	 \begin{equation}
	 \beta_u = \frac{dI(T;Y)}{dI(X;T)} \left( \frac{du(I(X;T))}{dI(X;T)} \right)^{-1} = \frac{\beta}{u'(I(X;T))},
	 \end{equation}
	 
	 where $\beta$ is the Lagrange multiplier from the IB Lagrangian \citep{tishby2000information} and $u'(I(X;T))$ is $\frac{du(I(X;T))}{dI(X;T)}$. Also, if we want to avoid indeterminations of $\beta_u$ we need $u'(I(X;T))$ not to be 0. Since we already imposed $u$ to be monotonically non-decreasing, we can solve this issue by strengthening this condition. That is, we will require \textbf{$u$ to be monotonically increasing}.
	 
	 We would like $\beta_u$ to be continuous, this way there would be a unique $\beta_u$ for each value of $I(X;T)$. We know $\beta$ is a non-increasing function of $I(X;T)$ (Lemma 6 of \citet{gilad2003information}). Hence, if we want $\beta_u$ \textbf{to be a strictly decreasing function of} $I(X;T)$, we will require $u'$ to be {a} strictly increasing function of $I(X;T)$. Therefore, \textbf{we will require $u$ to be a strictly convex function}.
	 
	  Thus, if $u$ is {a} strictly convex and monotonically increasing function, for each point $(I(X;T),I(T;Y))$ in the IB curve s.t. $dI(T;Y)/dI(X;T) > 0$ there is a unique $\beta_u$ for which maximizing $\mathcal{L}_{\textnormal{IB},\textnormal{u}}(T;{\beta_u})$ achieves this solution.
\end{proof}

\section{Proof of Proposition \ref{prop:bijective_mapping_beta_ixt}}
\label{proof:bijective_mapping}

\begin{proof}
	In Theorem \ref{th:ib_convex_lagrangians} we showed how each point of the IB curve $(I(X;T),I(T;Y))$ can be found with a unique $\beta_u$ maximizing $\mathcal{L}_{\textnormal{IB},u}(T;{\beta_u})$. Therefore, since we also proved  $\mathcal{L}_{\textnormal{IB},u}(T;{\beta_u})$ is strictly concave w.r.t. $T$ we can find the values of $\beta_u$ that maximize the Lagrangian for fixed $I(X;T)$. 
	
	First, we look at the solutions of the Lagrangian maximization:
	
	\begin{equation}
	\frac{d\mathcal{L}_{\textnormal{IB},u}(T;{\beta_u})}{dT} = \frac{d(f_{\textnormal{IB}}(I(X;T)) - \beta_u u(I(X;T)))}{dT} = \frac{df_{\textnormal{IB}}(I(X;T))}{dT} - \beta_u \frac{du(I(X;T))}{dI(X;T)} \frac{dI(X;T)}{dT} = 0.
	\end{equation}
	
	Then as before we can integrate at both sides for all $T \in \Delta$ and solve for $\beta_u$:
	
	\begin{equation}
	\beta_u = \frac{df_{\textnormal{IB}}(I(X;T))}{dI(X;T)} \frac{1}{u'(I(X;T))}.
	\label{eq:proof_bijective_mapping_betas_ixt_1}
	\end{equation}
	
	Moreover, since $u$ is a strictly convex function its derivative $u'$ is strictly {increasing}. Hence, $u'$ is an invertible function (since a strictly {increasing} function is bijective and a function is invertible iff it is bijective by definition). Now, if we consider $\beta_u > 0$ to be known and $I(X;T)$ to be the unknown we can solve for $I(X;T)$ and get:
	
	\begin{equation}
		I(X;T) = (u')^{-1} \left( \frac{df_{\textnormal{IB}}(I(X;T))}{dI(X;T)} \frac{1}{\beta_u} \right).
	\end{equation}
	
	Note we require $\beta_u$ not to be 0 so the mapping is defined.
\end{proof}

\section{Proof of Corollary \ref{cor:domain_conv_ib_lagrange}}
\label{proof:domain_beta_f_known}

\begin{proof}
We will start the proof by proving the following useful Lemma.

\begin{Lemma} Let $\mathcal{L}_{\textnormal{IB},u}(T;{\beta_u})$ be a convex IB Lagrangian, then $\sup_{T \in \Delta} \{ \mathcal{L}_{\textnormal{IB},u}(T;0) \}= I(X;Y)$.
\label{lemma:max_beta_0}
\end{Lemma}

\begin{proof}
	Since $\mathcal{L}_{\textnormal{IB},u}(T;0) = I(T;Y)$, maximizing this Lagrangian is directly maximizing $I(T;Y)$. We know $I(T;Y)$ is a concave function of $T$ for $T \in \Delta$ (Theorem 2.7.4 from \citet{cover2012elements}); hence it has a supremum. We also know $I(T;Y) \leq I(X;Y)$. Moreover, we know $I(X;Y)$ can be achieved if, for example, $Y$ is a deterministic function of $T$ (since then the Markov Chain $X \leftrightarrow T \leftrightarrow Y$ is formed). Thus, $\sup_{T \in \Delta} \{ \mathcal{L}_{\textnormal{IB},u}(T;0) \} = I(X;Y)$.
\end{proof}

	For $\beta_u = 0$ we know maximizing $\mathcal{L}_{\textnormal{IB},u}(T;0)$ we can obtain the point in the IB curve $(r_{\textnormal{max}},I_{\textnormal{max}})$ (Lemma \ref{lemma:max_beta_0}). 
	Moreover, we know that for every point $(I(X;T), f_{\textnormal{IB}}(I(X;T)))$ such that $d f_{\textnormal{IB}}(I(X;T)) / dI(X;T) > 0$, $\exists ! \beta_u$ s.t. $\max \{\mathcal{L}_{\textnormal{IB},u}(T;{\beta_u}) \}$ achieves that point (Theorem \ref{th:ib_convex_lagrangians}). Thus,
		$\exists ! \beta_{u,\textnormal{min}}$ s.t. $\lim_{r \rightarrow r_{\textnormal{max}}^{-}} (r, f_{\textnormal{IB}}(r))$ is achieved. From Proposition \ref{prop:bijective_mapping_beta_ixt} we know this $\beta_{u,\textnormal{min}}$ is given by
		
		\begin{equation}
		\beta_{u,\textnormal{min}} = \lim_{r \rightarrow r_{\textnormal{max}}^{-}} \left \{ \frac{f'_{\textnormal{IB}}(r)}{u'(r)}\right \}.
		\end{equation}
		
		Since we know $f_{\textnormal{IB}}(I(X;T))$ is a concave non-decreasing function in $(0,r_{\textnormal{max}})$ (Lemma 5 of \citet{gilad2003information}) we know it is continuous in this interval. In addition we know $\beta_u$ is strictly decreasing w.r.t. $I(X;T)$ (Theorem \ref{th:ib_convex_lagrangians}). Furthermore, by definition of $r_{\textnormal{max}}$ and knowing $I(T;Y) \leq  I(X;Y)$ we know $f'_{\textnormal{IB}}(r) = 0$, $\forall r > r_{\textnormal{max}}$. Therefore, we cannot ensure the exploration of the IB curve for $\beta_u'$ s.t. $0 < \beta_u' < \beta_{u,\textnormal{min}}$.
		
		Then, since $u$ is a strictly increasing function in $(0, r_{\textnormal{max}})$, $u'$ is positive in that interval. Hence, taking into account $\beta_u$ is strictly decreasing we can find a maximum $\beta_u$ when $I(X;T)$ approaches to 0. That is,
		
	\begin{equation}
	\beta_{u,\textnormal{max}} = \lim_{r \rightarrow 0^+} \left \{ \frac{f'_{\textnormal{IB}}(r)}{u'(r)}\right \},
\end{equation}
\end{proof}

\section{Proof of Corollary \ref{cor:bound_domain}}
\label{proof:bound_domain_beta}

\begin{proof}
If we use Corollary \ref{cor:domain_conv_ib_lagrange}, it is straightforward to see that $\beta_u \subseteq [L_-, L_+]$ if $\beta_{u,\textnormal{min}} \geq L_-$  and $\beta_{u,\textnormal{max}} \leq L_+$ for all IB curves $f_{\textnormal{IB}}$ and functions $u$. Therefore, we look at a domain bound dependent on the function choice. That is, if we can find $\beta_{\textnormal{min}} \leq f'_{\textnormal{IB}}(r)$ and $\beta_{\textnormal{max}} \geq f'_{\textnormal{IB}}(r)$ for all IB curves and all values of $r$, then 

\begin{equation}
	B_u \subseteq \left[\frac{\beta_{\textnormal{min}}}{\lim_{r \rightarrow r_{\textnormal{max}}^-} \lbrace u'(r) \rbrace}, \frac{\beta_{\textnormal{max}}}{\lim_{r \rightarrow 0^+} \lbrace u'(r) \rbrace}  \right].
\end{equation}

The region for all possible IB curves regardless of the relationship between $X$ and $Y$ is depicted in Figure \ref{fig:example_ib_curve}. The hard limits are imposed by the DPI (Theorem 2.8.1 from \citet{cover2012elements}) and the fact that the mutual information is non-negative (Corollary with Equation 2.90 for discrete and first Corollary of Theorem 8.6.1 for continuous random variables from \citet{cover2012elements}). Hence, a minimum and maximum values of $f'_{\textnormal{IB}}$ are given by the minimum and maximum values of the slope of the Pareto frontier. Which means

\begin{equation}
	B_u \subseteq \left[0, \frac{1}{\lim_{r \rightarrow 0^+} \lbrace u'(r) \rbrace}  \right].
\end{equation}

Note $0/({\lim_{r \rightarrow r_{\textnormal{max}}^-} \lbrace u'(r) \rbrace}) = 0$ since $u$ is monotonically increasing and, thus, $u'$ will never be 0.

\begin{figure}
\centering
\includegraphics[width=0.7\linewidth]{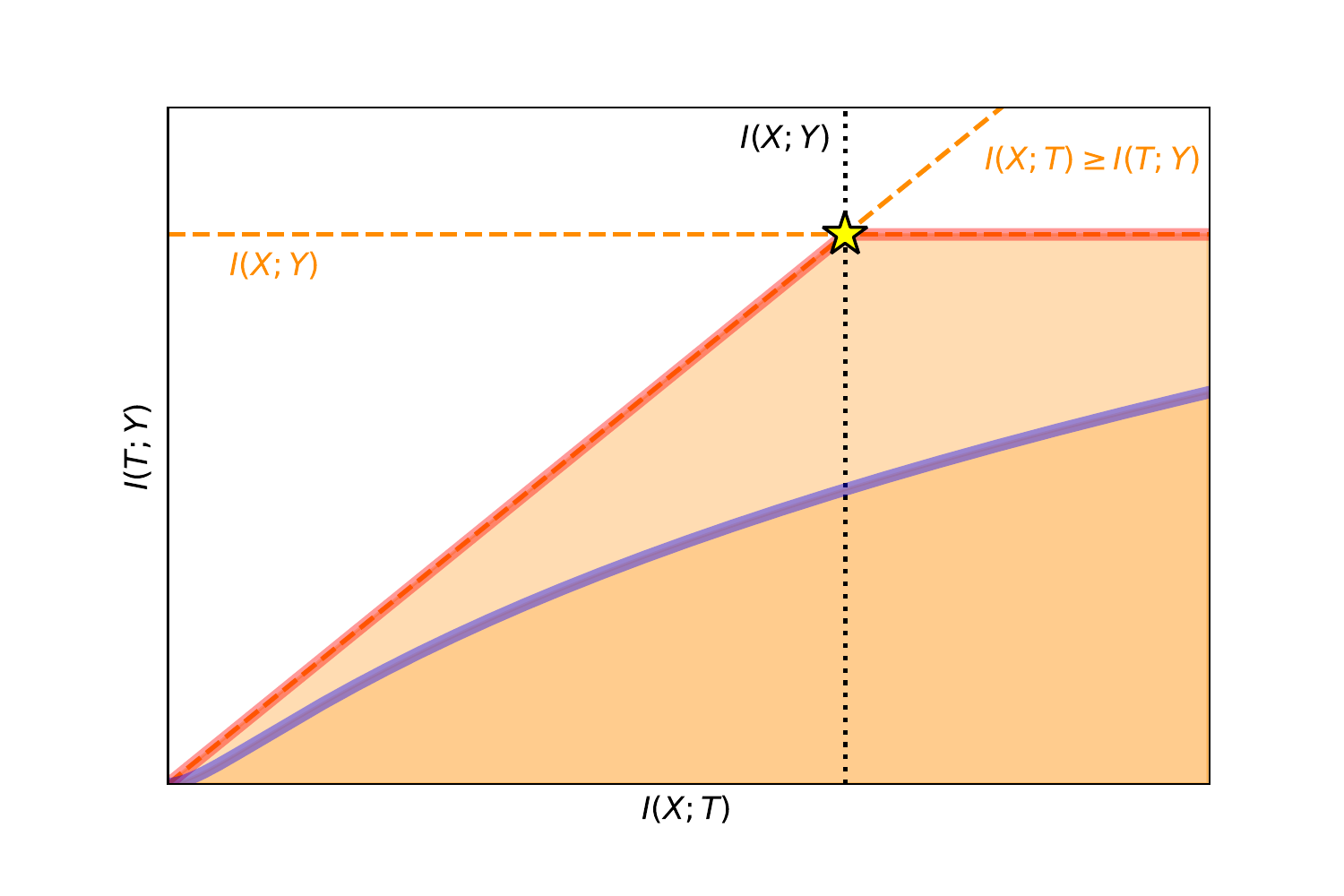}
\caption{Graphical representation of the IB curve in the information plane. Dashed lines
in orange represent tight bounds confining the region (in light orange) of possible IB curves
(delimited by the red line, also known as the Pareto frontier). Black dotted
lines are informative values. In blue we show an example of a possible IB curve confining
a region (in darker orange) of an IB curve which does not achieve the Pareto frontier. Finally, the
yellow star represents the point where the representation keeps the same information about the input and the output.} 
\label{fig:example_ib_curve}
\end{figure}

Then, we can tighten the bound using the results from \citet{wu2019learnability}, where, in Theorem 2, they showed the slope of the Pareto frontier could be bounded in the origin by $f'_{\textnormal{IB}} \leq (\inf_{\Omega_x \subset \mathcal{X}} \lbrace \beta_0(\Omega_x)  \rbrace)^{-1}$. Finally, we know that in deterministic classification tasks $\inf_{\Omega_x \subset \mathcal{X}} \lbrace \beta_0(\Omega_x)  \rbrace = 1$, which aligns with \citet{kolchinsky2018caveats} and what we can observe from Figure \ref{fig:example_ib_curve}. Therefore, 

\begin{equation}
	B_u \subseteq \left[0, \frac{(\inf_{\Omega_x \subset \mathcal{X}} \lbrace \beta_0(\Omega_x)  \rbrace)^{-1}}{\lim_{r \rightarrow 0^+} \lbrace u'(r) \rbrace}  \right] \subseteq \left[0, \frac{1}{\lim_{r \rightarrow 0^+} \lbrace u'(r) \rbrace}  \right].
\end{equation}
\end{proof}

\section{Other Lagrangian Families}
\label{app:alternatives_to_convex_ib_lagrangians}

We can use the same ideas we used for the convex IB Lagrangian to formulate new families of Lagrangians that allow the exploration of the IB curve. For that we will use the duality of the IB curve (Lemma 10 of \citep{gilad2003information}). That is:

\begin{Definition} [IB dual functional]
	\label{def:ib_dual_f}
	Let $X$ and $Y$ be statistically dependent variables. Let also $\Delta$ be the set of random variables $T$ obeying the Markov condition $Y \leftrightarrow X \leftrightarrow T$. Then the IB dual functional is

	\begin{equation}
	F_{\textnormal{IB},\textnormal{min}}(i) = \min_{T \in \Delta} \left\{ I(X;T) \right\} \ \textnormal{s.t.} \ I(T;Y) \geq i, \ \forall i \in [0, I(X;Y)).
	\label{eq:ib_optim_min}
\end{equation}

\end{Definition}

\begin{Theorem}[IB curve duality] Let the IB curve be defined by the solutions of $F_{\textnormal{IB,max}}(r)$ for varying $r \in [0,\infty)$. Then, 

\begin{equation}
	\forall r \exists i \text{ s.t. } (r, F_{\textnormal{IB,max}}(r)) = (F_{\textnormal{IB,min}}(i), i)
\end{equation}

and 

\begin{equation}
	\forall i \exists r \text{ s.t. } (F_{\textnormal{IB,min}}(i), i) = (r, F_{\textnormal{IB,max}}(r)).
\end{equation}

\end{Theorem}

From this definition it follows that minimizing the \textit{dual IB Lagrangian}, $
	\mathcal{L}_{\textnormal{IB,dual}}(T;{\beta_{\textnormal{dual}}}) = I(X;T) - \beta_{\textnormal{dual}} I(T;Y) $, for $\beta_{\textnormal{dual}} = \beta^{-1}$ is equivalent to maximizing the IB Lagrangian. In fact, the original Lagrangian for solving the problem was defined this way \citep{tishby2000information}. We decided to use the maximization version because the domain of useful $\beta$ is bounded while it is not for $\beta_{\textnormal{dual}}$.

Following the same reasoning as we did in the proof of Theorem \ref{th:ib_convex_lagrangians}, we can ensure the IB curve can be explored if:

\begin{enumerate}
\item We minimize the concave IB Lagrangian $\mathcal{L}_{\textnormal{IB},v}(T;{\beta_v}) = I(X;T) - \beta_v v(I(T;Y))$.
\item We maximize the dual concave IB Lagrangian $\mathcal{L}_{\textnormal{IB},v,\textnormal{\textnormal{dual}}}(T;{\beta_{v,\textnormal{dual}}}) = v(I(T;Y)) - \beta_{v,\textnormal{dual}} I(X;T)$.
\item We minimize the dual convex IB Lagrangian $\mathcal{L}_{\textnormal{IB},u,\textnormal{dual}}(T;{\beta_{u,\textnormal{dual}}}) = u(I(X;T)) - \beta_{u,\textnormal{dual}} I(T;Y)$.
\end{enumerate}

Here, $u$ is a monotonically increasing strictly convex function, $v$ is a monotonically increasing strictly concave function, and $\beta_v, \beta_{v,\textnormal{dual}}, \beta_{u,\textnormal{dual}}$ are the Lagrange multipliers of the families of Lagrangians defined above.

In a similar manner, one could obtain relationships between the Lagrange multipliers of the IB Lagrangian and the convex IB Lagrangian with these Lagrangian families. For instance, the convex IB Lagrangian $\mathcal{L}_{\textnormal{IB},u}(T;\beta_u)$ is related with the concave IB Lagrangian $\mathcal{L}_{\textnormal{IB},v}(T;\beta_v)$ as defined by Propositon \ref{prop:relation_conv_conc}.

\begin{Proposition}[Relationship between the convex and concave IB Lagrangians] Consider the convex and concave IB Lagrangians $\mathcal{L}_{\textnormal{IB},u}(T;\beta_u)$, $\mathcal{L}_{\textnormal{IB},v}(T;\beta_v)$. Let the IB curve defined as in Definition \ref{def:ib_curve} be $f_{\textnormal{IB}}$. Then, if we fix the functions $u$ and $v$ we can obtain the same point in the IB curve $(r, f_{\textnormal{IB}}(r))$ with both Lagrangians when
\begin{equation}
	\beta_v^{-1} = f_{\textnormal{IB}}'(r) v' \left( f_{\textnormal{IB}}\left( (u')^{-1}\left( \frac{f_{\textnormal{IB}}'(r)}{\beta_u}\right)\right)\right),
\label{eq:beta_v_with_beta_u}
\end{equation}
or equivalently,
\begin{equation}
	\beta_u^{-1} = \frac{1}{f_{\textnormal{IB}}'(r)}u' \left( f_{\textnormal{IB}}^{-1} \left( (v')^{-1} \left( \frac{\beta_v^{-1}}{f_{\textnormal{IB}}'(r)} \right) \right)\right).
\label{eq:beta_u_with_beta_v}
\end{equation}
\label{prop:relation_conv_conc}
\end{Proposition}

\begin{proof}
If we proceed like we did in the proof of Proposition \ref{prop:bijective_mapping_beta_ixt} we can find the mapping between $I(X;T)$ and $\beta_u$ and between $I(T;Y)$ and $\beta_v$. That is,
\begin{equation}
I(X;T) = (u')^{-1} \left( \frac{df_{\textnormal{IB}}(I(X;T))}{dI(X;T)} \frac{1}{\beta_u} \right) \textnormal{ and } I(T;Y) = (v')^{-1} \left( \left( \frac{df_{\textnormal{IB}}(I(X;T))}{dI(X;T)}\right)^{-1} \frac{1}{\beta_v} \right).
\end{equation}
Then, if we recall that $I(T;Y) = f_{\textnormal{IB}}(I(X;T))$, we can directly obtain that
\begin{equation}
f_{\textnormal{IB}}\left( (u')^{-1} \left( \frac{df_{\textnormal{IB}}(I(X;T))}{dI(X;T)} \frac{1}{\beta_u} \right) \right) = (v')^{-1} \left( \left( \frac{df_{\textnormal{IB}}(I(X;T))}{dI(X;T)}\right)^{-1} \frac{1}{\beta_v} \right).
\label{eq:equation_equivalence}
\end{equation}
Then, if we solve Equation (\ref{eq:equation_equivalence}) with a fixed point $(I(X;T) = r, I(T;Y) = f_{\textnormal{IB}}(r))$ for $\beta_v$ we obtain Equation  (\ref{eq:beta_v_with_beta_u}), and if we solve it for $\beta_u$ we obtain Equation (\ref{eq:beta_u_with_beta_v}).
\end{proof}

Also, one could find a range of values for these Lagrangians to allow for the IB curve exploration and define a bijective mapping between their Lagrange multipliers and the IB curve. However, (i) as mentioned in Section \ref{subsec:why_the_ib}, $I(T;Y)$ is particularly interesting to maximize without transformations because of its meaning. Moreover, (ii) like $\beta_{\textnormal{dual}}$, the domain of useful $\beta_v$ and $\beta_{u,\textnormal{dual}}$ is not upper bounded.  These two reasons make these other Lagrangians less preferable. We only include them here for completeness. Nonetheless, we encourage the curiours reader to explore these families of Lagrangians too. For example, a possible interesting research would be investigating if some particularization of the concave IB Lagrangian suffers from an issue like value convergence that can be exploited for approximately obtaining any predictability level $I(T;Y) = i^*$ for many values of $\beta_v$.

\section{Experimental setup details and further experiments}
\label{app:experimental_setup_details}

{In order to generate empirical support for our claims we performed several experiments on different datasets with different neural network architectures and different ways of calculating the information bottleneck.}

\subsection{Information bottleneck calculations}

{The information bottleneck is calculated modifying the nonlinear-IB \citep{kolchinsky2017nonlinear}. This method of calculating the information bottleneck is a neural network that minimizes the cross-entropy while also miniminizing an upper bound estimate of the mutual information $I_{\theta} \approx I(X;T)$. The nonlinear-IB relies on a kernel-based estimate of this mutual information \citep{kolchinsky2017estimating}. We modify this calculation method by applying the function $u$ to the $I(X;T)$ estimate.}

{For the nonlinear-IB calculations we estimated the gradients of both $I_{\theta}(X;T)$ and the cross entropy with the same mini-batch. Moreover, we did not learn the covariance of the mixture of Gaussians used for the kernel density estimation of $I_{\theta}(X;T)$ and we set it to $(\exp(-1))^2$.}

{In both methods, and for all the experiments, we assumed a Gaussian stochastic encoder $T = f_{\textnormal{enc}}(X;\theta) + W$ with $p_W = \mathcal{N}(0,I_d)$, where $d$ are the number of dimensions of the representations. We trained the neural networks with the Adam optimization algorithm \citep{kingma2014adam} with a learning rate of $10^{-4}$ and a $0.6$ decay rate every 10 epochs. We used a batch size of 128 samples and all the weights were initialized according to the method described by \citet{glorot2010understanding} using a Gaussian distribution.}

{Then, we used the DBSCAN algorithm \citep{ester1996density, schubert2017dbscan} for clustering. Particularly, we used the scikit-learn \citep{pedregosa2011scikit} implementation with $\epsilon = 0.3$ and \texttt{min\_samples} = 50.}

{The reader can find the PyTorch \citep{paszke2017automatic} implementation in the following link: \href{https://github.com/burklight/convex-IB-Lagrangian-PyTorch}{\texttt{https://github.com/burklight/convex-IB-Lagrangian-PyTorch}}.}

\subsection{The experiments}

{We performed experiments in four different datasets:}

\begin{itemize}

\item {\textbf{A classification task on the MNIST dataset \citep{lecun1998gradient}} (Figures \ref{fig:example_performance}, \ref{fig:clusters_alpha_1}, \ref{fig:example_performance_power}, \ref{fig:example_performance_exponential} and \ref{fig:example_clusters_alphas_etas} and top row from Figure \ref{fig:value_convergence}). This dataset contains 60,000 training samples and 10,000 testing samples of hand-written digits. The samples are 28x28 pixels and are labeled from 0 to 9; i.e., $\mathcal{X} = \mathbb{R}^{784}$ and $\mathcal{Y} = \lbrace 0, 1, ..., 9 \rbrace$. The data is pre-processed so that the input has zero mean and unit variance. This is a deterministic setting, hence the experiment is designed to showcase how the convex IB Lagrangians allow to explore the IB curve in a setting where the normal IB Lagrangian cannot and the relationship between the performance plateaus and the clusterization phenomena. Furthermore, it intends to showcase the behavior of the power and exponential Lagrangians with different parameters of $\alpha$ and $\eta$. Finally, it wants to demonstrate how the value convergence can be employed to approximately obtain a specific compression value. In this experiment, the encoder $f_{\textnormal{enc}}$ is a three fully-connected layer encoder with 800 ReLU units on the first two layers and 2 linear units on the last layer ($T \in \mathbb{R}^2$), and the decoder $f_{\textnormal{dec}}$ is a fully-conected 800 ReLU unit layers followed by an output layer with 10 softmax units. The convex IB Lagrangian was calculated using the nonlinear-IB.}

{
In Figure \ref{fig:example_performance_power} we show how the IB curve can be explored with different values of $\alpha$ for the power IB Lagrangian and in Figure \ref{fig:example_performance_exponential} for different values of $\eta$ and the exponential IB Lagrangian.}

{
Finally, in Figure \ref{fig:example_clusters_alphas_etas} we show the clusterization for the same values of $\alpha$ and $\eta$ as in Figures \ref{fig:example_performance_power} and \ref{fig:example_performance_exponential}. In this way the connection between the performance discontinuities and the clusterization is more evident. Furthermore, we can also observe how the exponential IB Lagrangian maintains better the theoretical performance than the power IB Lagrangian (see Appendix \ref{app:guidelines_on_choosing_proper_h} for an explanation of why).}
{\item \textbf{A classification task on the Fashion-MNIST dataset \citep{xiao2017fashion}} (Figure \ref{fig:example_performance_fashion_mnist}). As MNSIT, this dataset contains 60,000 training and 10,000 testing samples of 28x28 pixel images labeled from 0 to 9 and constitutes a deterministic setting. The difference is that this dataset contains fashion products instead of hand-written digits and it represents a harder classification task \citep{xiao2017fashion}. The data is also pre-processed so that the input has zero mean and unit variance. For this experiment, the encoder $f_{\textnormal{enc}}$ is composed by a 2-layer convolutional neural network (CNN) with 32 filters on the first layer and 128 filters on the second with kernels of size 5 and stride 2. This CNN is followed by two fully-connected layers of 128 linear units ($T \in \mathbb{R}^{128}$). After the first convolution and the first fully-connected layer a ReLU activation is employed. The decoder $f_{\textnormal{dec}}$ is a fully-connected 128 ReLU unit layer followed by an output layer with 10 softmax units. The convex IB Lagrangian was calculated using the nonlinear-IB. Therefore, this experiment intends to showcase how the convex IB Lagrangian can explore the IB curve for different neural network architectures and harder datasets.}

\item {\textbf{A regression task on the California housing dataset \citep{pace1997sparse}} (Figure \ref{fig:example_performance_california_housing}). This dataset contains 20,640 samples of 8 real number input variables like the longitude and latitude of the house (i.e., $X \in \mathbb{R}^8$) and a task output real variable representing the price of the house (i.e., $Y \in \mathbb{R}$). We used the log-transformed house price as the target variable and dropped the  992 samples in which the house price was equal or greater than \$$500,000$ so that the output distribution was closer to a Gaussian as they did in \citep{kolchinsky2017nonlinear}. The input variables were processed so that they had zero mean and unit variance and we randomly splitted the  samples into a 70\% training and 30\% test dataset. As in \citep{kolchinsky2017estimating}, for regression tasks we approximate $H(Y)$ with the entropy of a Gaussian with variance $\textnormal{Var}(Y)$ and $H(Y|T)$ with the entropy of a Gaussian with variance equal to the mean-squared error (MSE). This leads to the estimate $I(T;Y) \approx 0.5 \log(\textnormal{Var}(Y)/MSE)$. The encoder $f_{\textnormal{enc}}$ is a three fully-connected layer encoder with 128 ReLU units on the first two layers and 2 linear units on the last layer ($T \in \mathbb{R}^2$), and the decoder $f_{\textnormal{dec}}$ is a fully-conected 128 ReLU unit layers followed by an output layer with 1 linear unit. The convex IB Lagrangian was calculated using the nonlinear-IB. Hence, this experiment was designed to showcase the convex IB Lagrangian can explore the IB curve in stochastic scenarios for regression tasks.}

{
\item \textbf{A classification task on the TREC-6 dataset \citep{li2002learning}} (Figure \ref{fig:example_performance_trec} and bottom row from Figure \ref{fig:value_convergence}). This dataset is the 6 classes version of the TREC \citep{voorhees2000building} dataset. It contains 5,452 training and 500 test samples of text questions. Each question is labeled within 6 different semantic categories based on what the answer is; namely: Abbreviation, description and abstract concepts, entities, human beings, locations and numeric values. This dataset does not constitute a deterministic setting, since there are examples that could belong to more than one class and there are examples which are wrongly labeled (e.g., "\textit{What is a fear of parasites?}" could belong both to the description and abstract concept category, however it is labeled into the entity category), and hence $H(Y|X) > 0$. Following \href{https://github.com/bentrevett/pytorch-sentiment-analysis/blob/master/5\%20-\%20Multi-class\%20Sentiment\%20Analysis.ipynb}{this example} the encoder $f_{\textnormal{enc}}$ is composed by a \href{https://nlp.stanford.edu/projects/glove/}{6 billion token pre-trained 100-dimensional Glove word embedding} \citep{pennington2014glove}, followed by a concatenation of 3 convolutions with kernel sizes 2, 3 and 4 respectively, and finalized with a fully-connected 128 linear unit layer ($T \in \mathbb{R}^{128}$). The decoder $f_{\textnormal{dec}}$ is a single fully-connected 6 softmax unit layer. The convex IB Lagrangian was calculated using the nonlinear-IB. Thus, this experiment intends to show an example where the classification task does not convey a deterministic scenario, that the convex IB Lagrangian can recover the IB curve in complex stochastic tasks with complex neural network architectures and that the value convergence can be employed to obtain a specific compression value even in stochastic settings where the IB curve is unkown.}
\end{itemize}



\begin{figure}
\centering
\includegraphics[width=\linewidth]{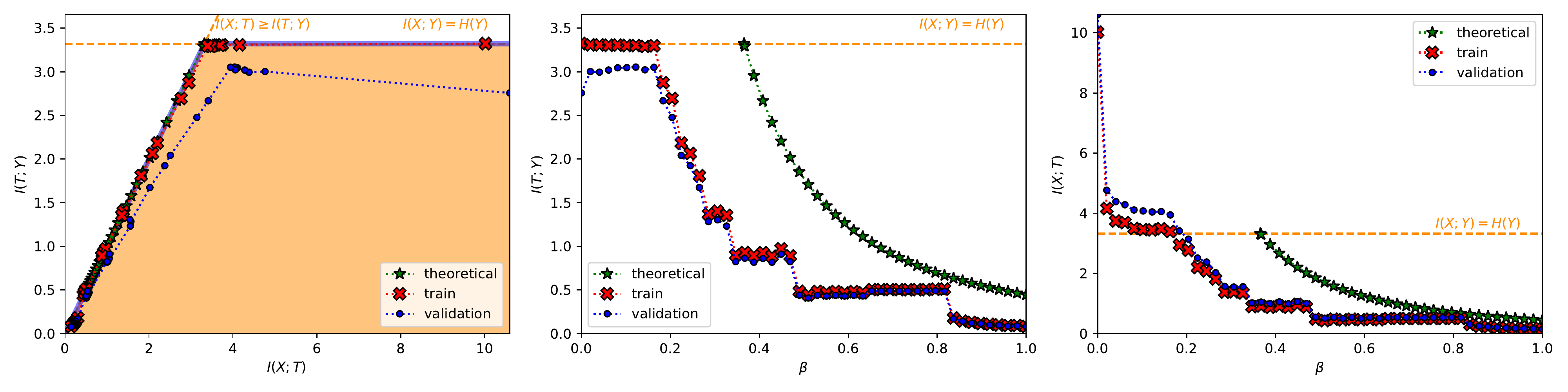}
\includegraphics[width=\linewidth]{figures/example_power_1.pdf}
\includegraphics[width=\linewidth]{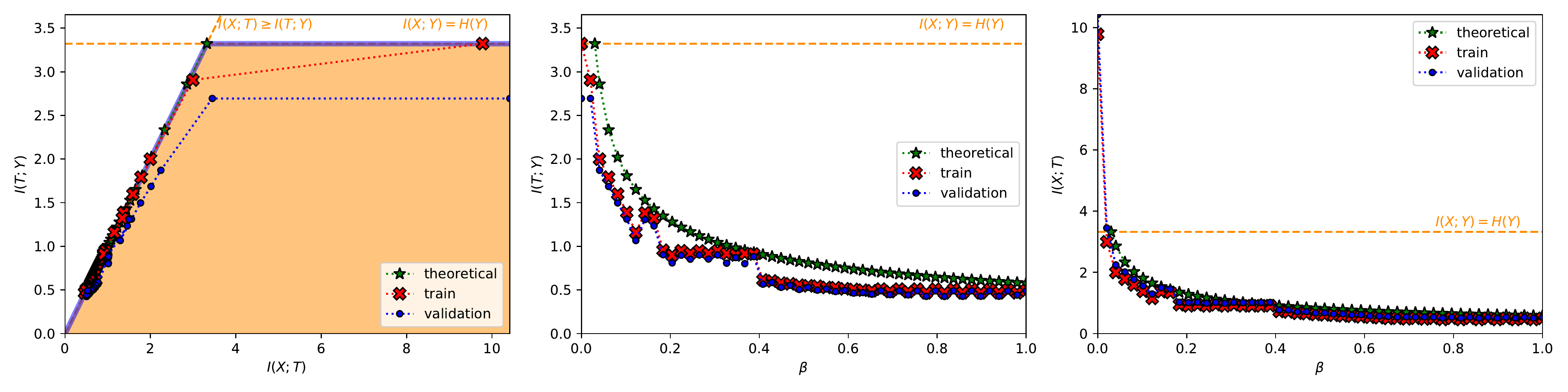}
\caption{Results for the power IB Lagrangian {in the MNIST dataset} with $\alpha = \lbrace 0.5, 1, 2 \rbrace$, from top to bottom. In each row, from left to right it is shown (i) the information plane, where the region of possible solutions of the IB problem is shadowed in light orange and the information-theoretic limits are the dashed orange line; (ii) $I(T;Y)$ as a function of $\beta_u$; and (iii) the compression $I(X;T)$ as a function of $\beta_u$. In all plots{,} the red crosses joined by a dotted line represent the values computed with the training set, the blue dots the values computed with the validation set and the green stars the theoretical values computed as dictated by Proposition \ref{prop:bijective_mapping_beta_ixt}. Moreover, in all plots{,} it is indicated $I(X;Y) = H(Y) = \log_2(10)$ in a dashed, orange line. All values are shown in bits.} 
\label{fig:example_performance_power}
\end{figure}

\begin{figure}
\centering
\includegraphics[width=\linewidth]{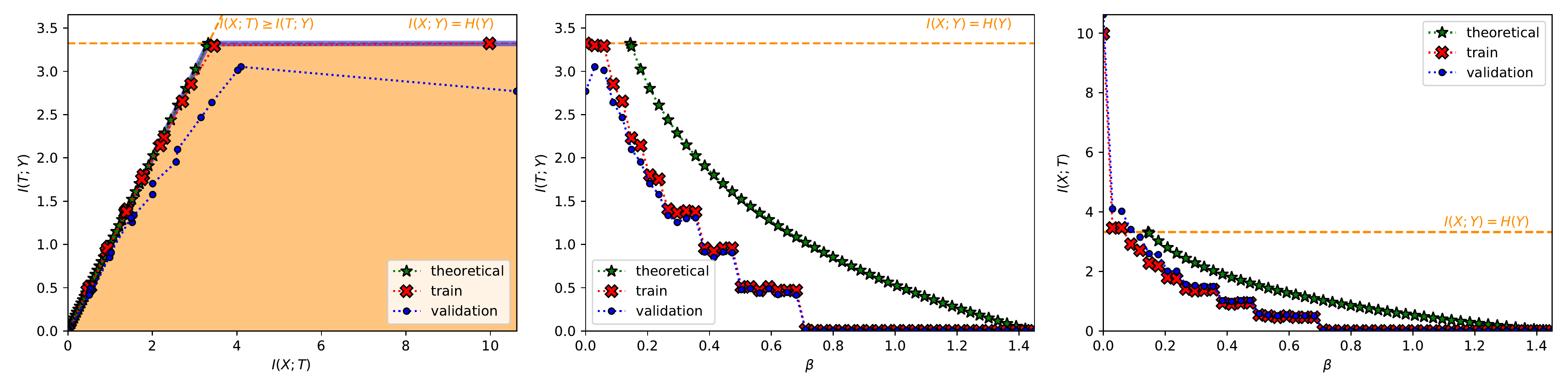}
\includegraphics[width=\linewidth]{figures/example_exp_1.pdf}
\includegraphics[width=\linewidth]{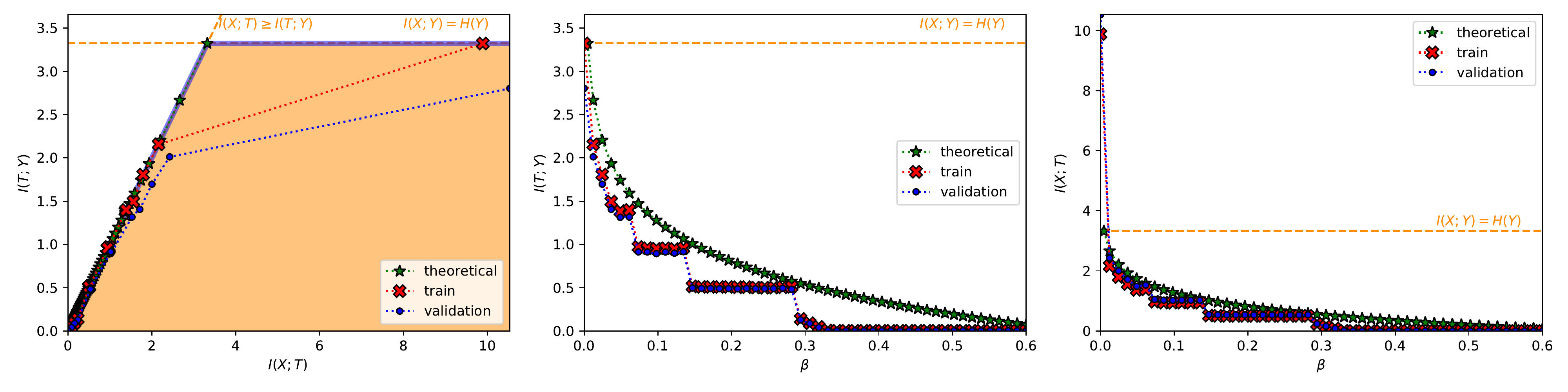}
\caption{Results for the exponential IB Lagrangian {in the MNIST dataset} with $\eta = \lbrace \log(2), 1, 1.5 \rbrace$, from top to bottom. In each row, from left to right it is shown (i) the information plane, where the region of possible solutions of the IB problem is shadowed in light orange and the information-theoretic limits are the dashed orange line; (ii) $I(T;Y)$ as a function of $\beta_u$; and (iii) the compression $I(X;T)$ as a function of $\beta_u$. In all plots{,} the red crosses joined by a dotted line represent the values computed with the training set, the blue dots the values computed with the validation set and the gren stars the theoretical values computed as dictated by Proposition \ref{prop:bijective_mapping_beta_ixt}. Moreover, in all plots{,} it is indicated $I(X;Y) = H(Y) = \log_2(10)$ in a dashed, orange line. All values are shown in bits.} 
\label{fig:example_performance_exponential}
\end{figure}

\begin{figure}
\centering
\includegraphics[width=0.32\linewidth]{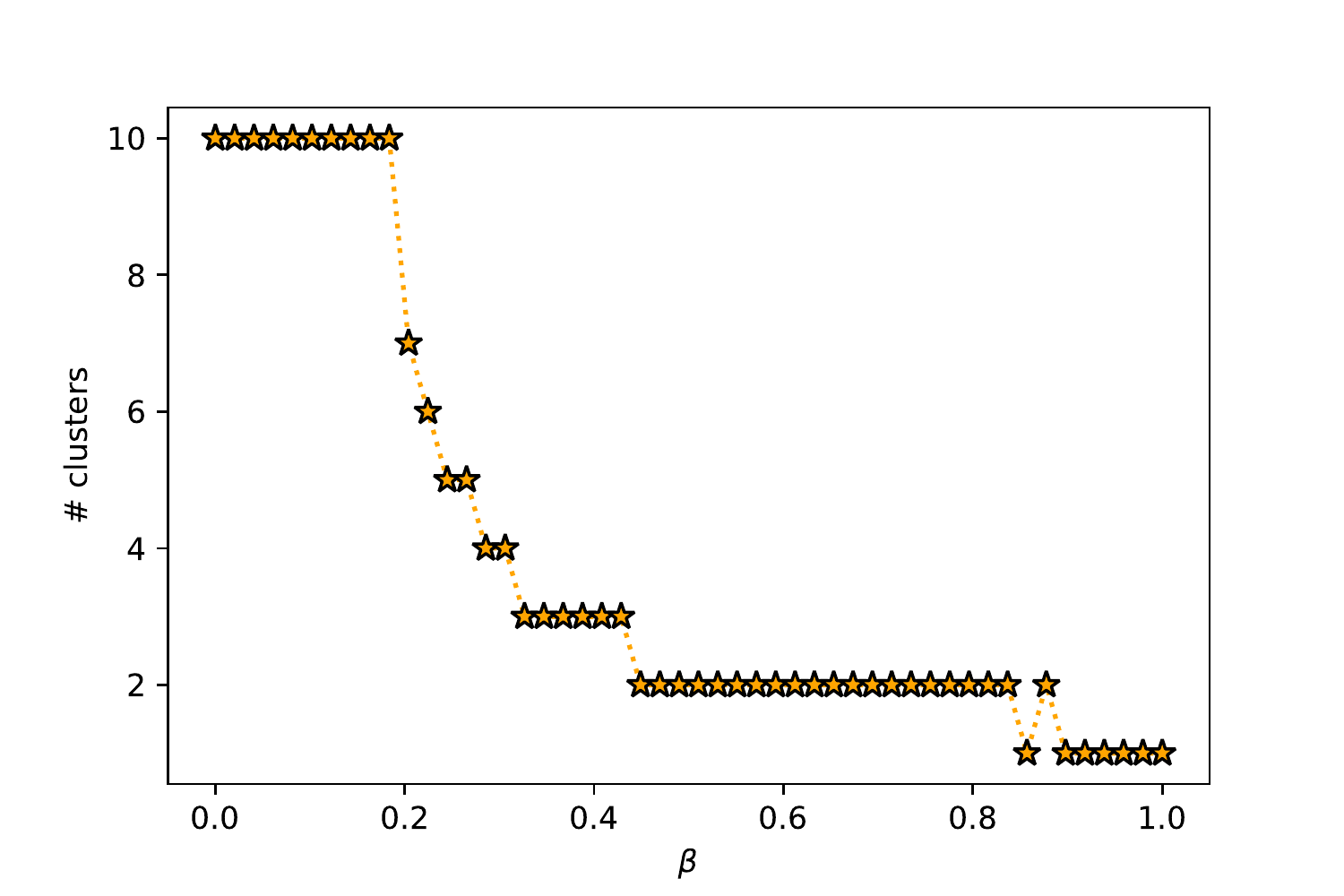} 
\includegraphics[width=0.32\linewidth]{figures/example_number_clusters_power_alpha_1.pdf}
\includegraphics[width=0.32\linewidth]{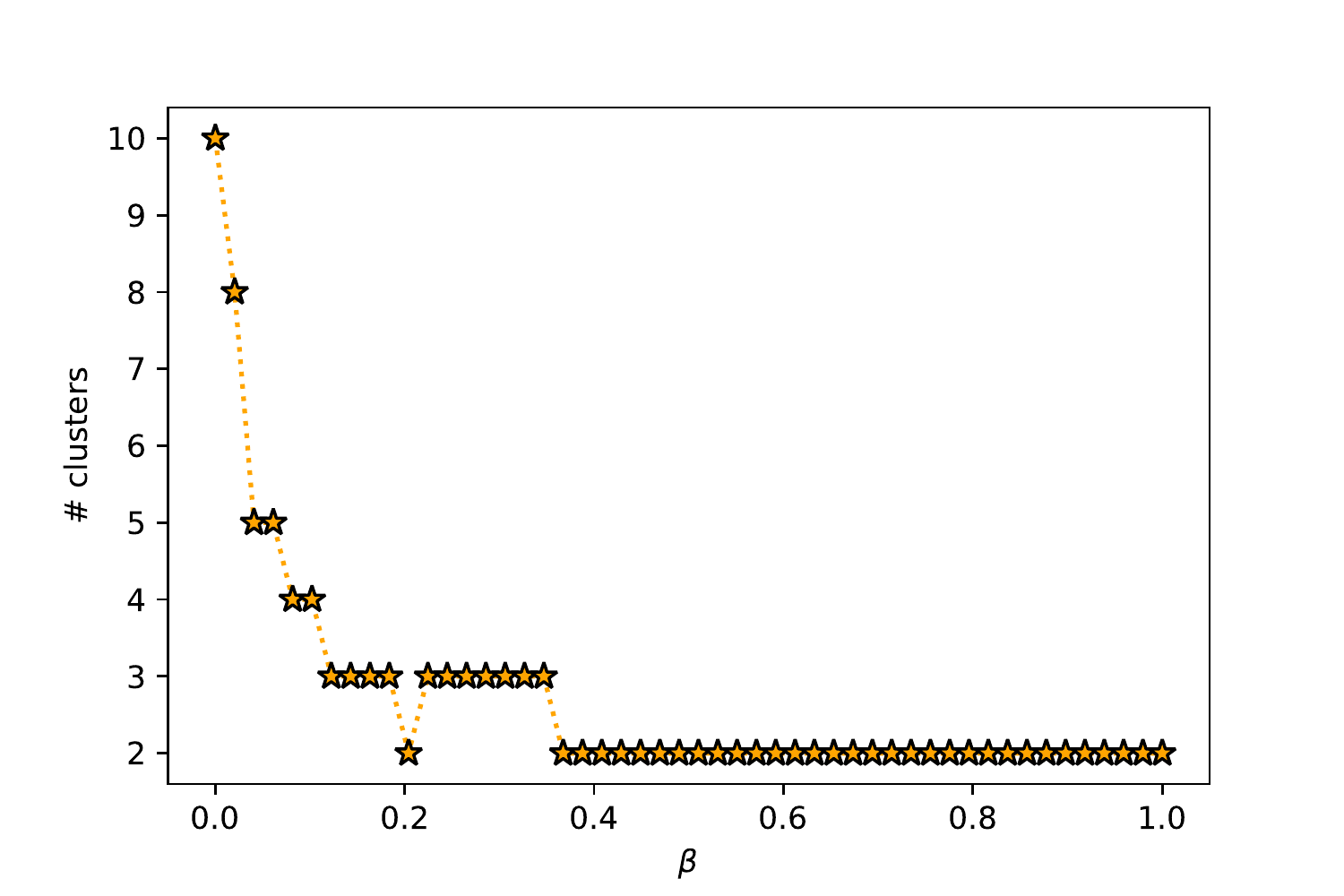}  
\includegraphics[width=0.32\linewidth]{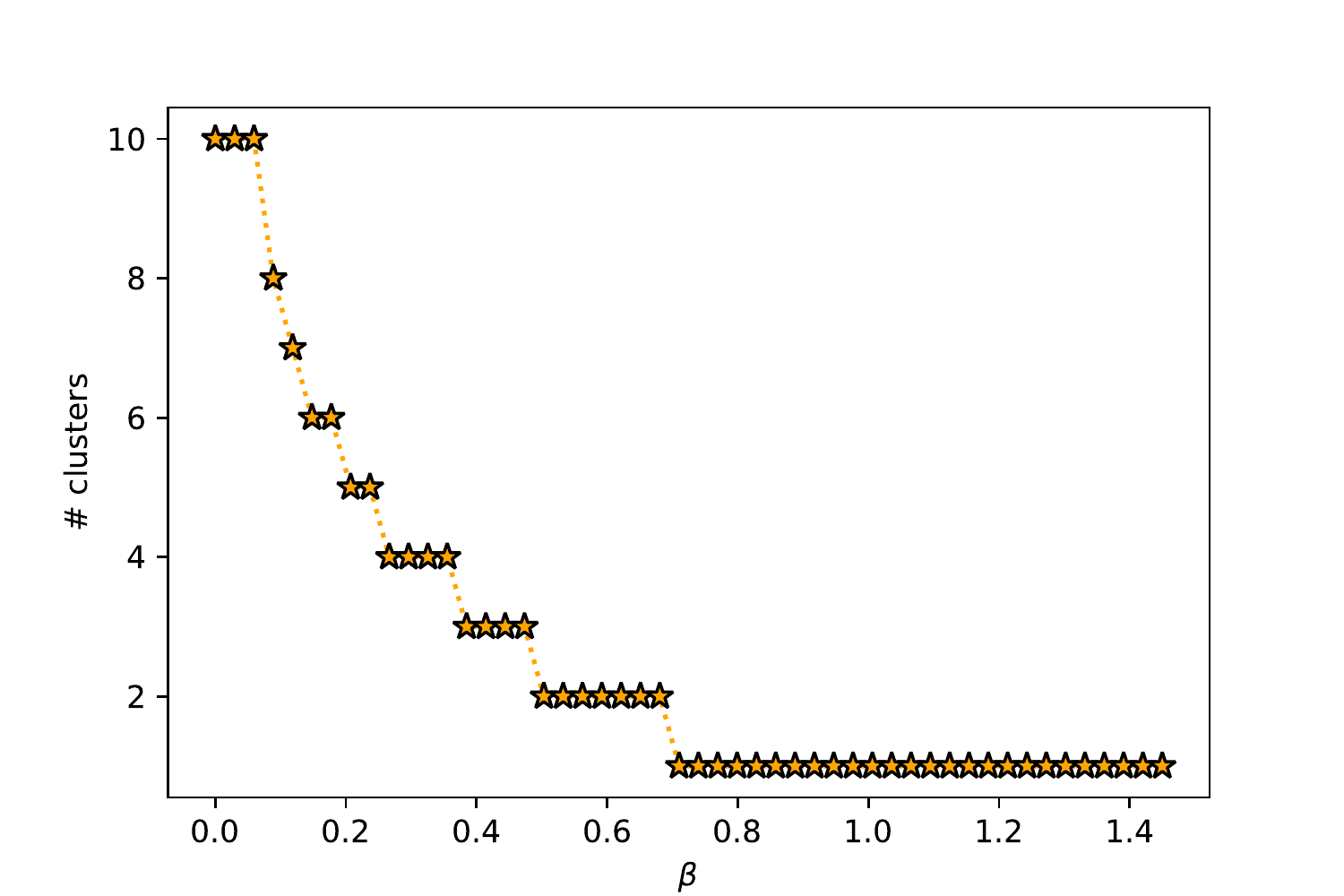}
\includegraphics[width=0.32\linewidth]{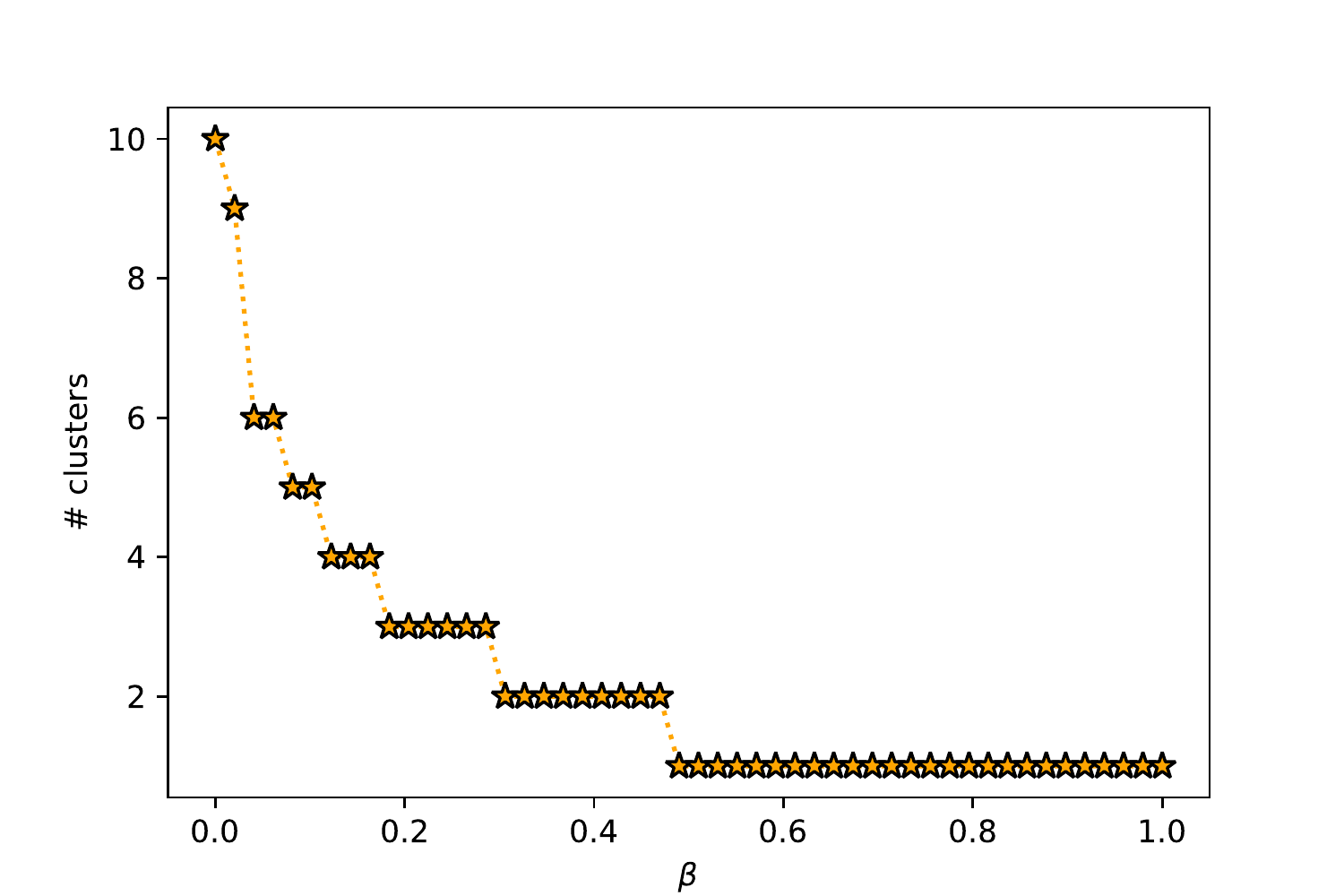}
\includegraphics[width=0.32\linewidth]{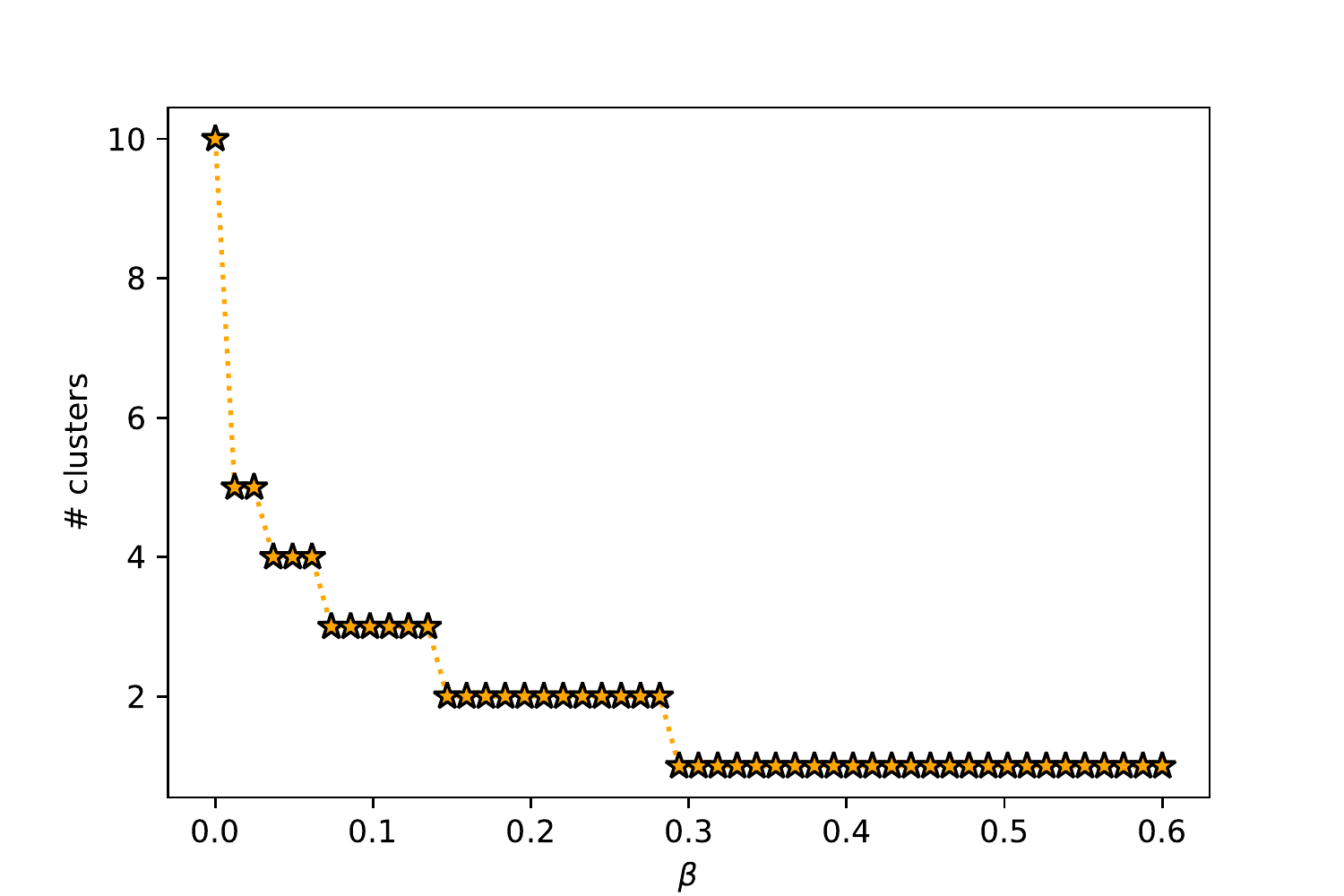}

\caption{Depiction of the clusterization behavior of the bottleneck variable {in the MNIST dataset}. In the first row, from left to right, the power IB Lagrangian with different values of $\alpha = \lbrace 0.5, 1, 2 \rbrace$. In the second row, from left to right, the exponential IB Lagrangian with different values of $\eta = \lbrace \log(2), 1, 1.5 \rbrace$.} 
\label{fig:example_clusters_alphas_etas}
\end{figure}

\begin{figure}
\centering
\includegraphics[width=\linewidth]{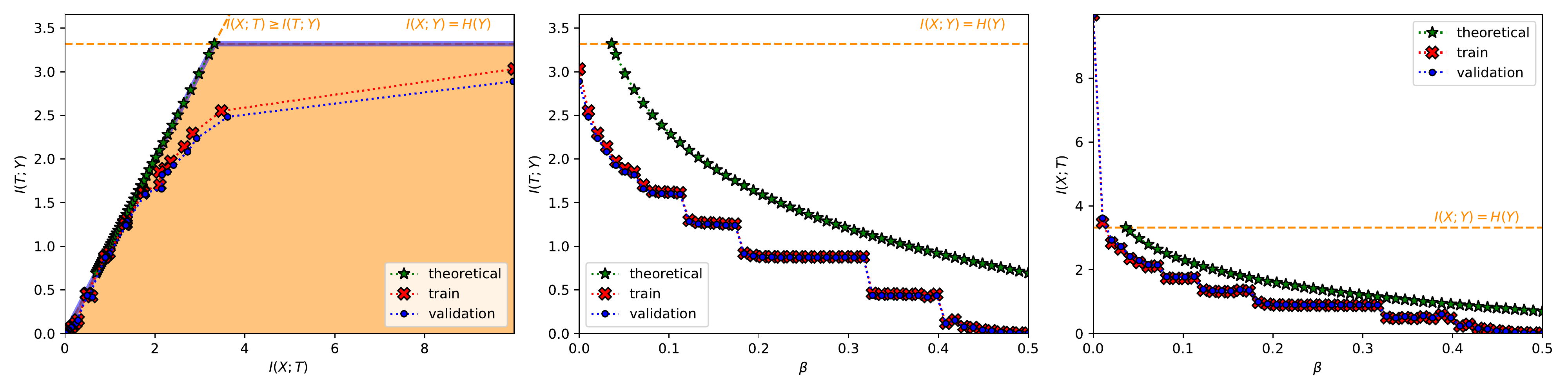}
\caption{{Results for the exponential IB Lagrangian in the Fashion MNIST dataset with $\eta = 1$. From left to right it is shown (i) the information plane, where the region of possible solutions of the IB problem is shadowed in light orange and the information-theoretic limits are the dashed orange line; (ii) $I(T;Y)$ as a function of $\beta_u$; and (iii) the compression $I(X;T)$ as a function of $\beta_u$. In all plots, the red crosses joined by a dotted line represent the values computed with the training set and the blue dots the values computed with the validation set. Moreover, in all plots, it is indicated $I(X;Y) = H(Y) = \log_2(10)$. All values are shown in bits.}}
\label{fig:example_performance_fashion_mnist}
\end{figure}

\begin{figure}
\centering
\includegraphics[width=\linewidth]{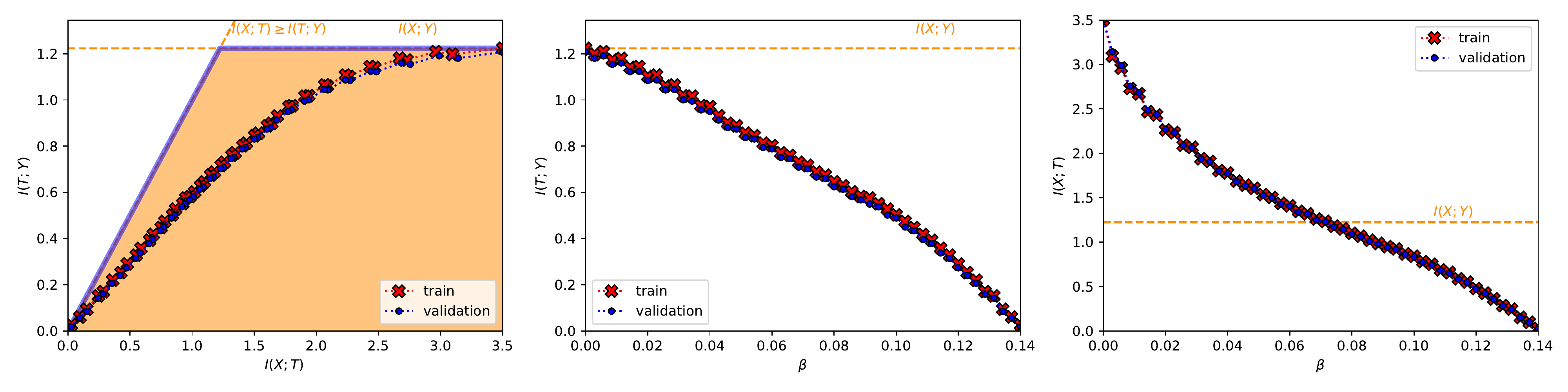}
\includegraphics[width=\linewidth]{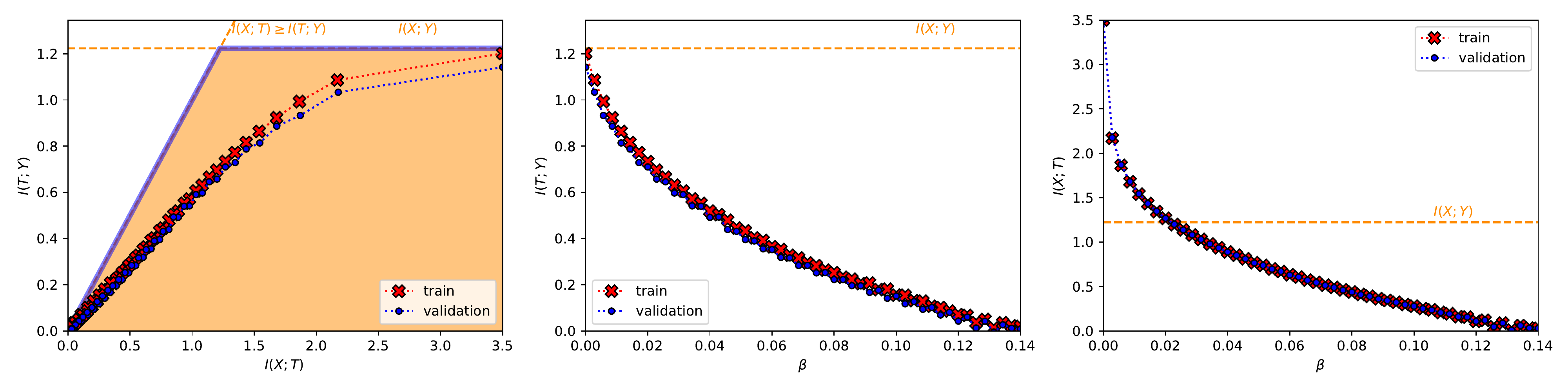}
\caption{{The top row shows the results for the normal IB Lagrangian, and the bottom row for the exponential IB Lagrangian with $\eta = 1$, both in the California housing dataset. In each row, from left to right it is shown (i) the information plane, where the region of possible solutions of the IB problem is shadowed in light orange and the information-theoretic limits are the dashed orange line; (ii) $I(T;Y)$ as a function of $\beta_u$; and (iii) the compression $I(X;T)$ as a function of $\beta_u$. In all plots, the red crosses joined by a dotted line represent the values computed with the training set and the blue dots the values computed with the validation set. Moreover, in all plots, it is indicated $I(X;Y)$ as the empirical value obtained maximizing $I(T;Y)$ without compression limitations as in \citep{kolchinsky2017nonlinear}. All values are shown in bits.}}
\label{fig:example_performance_california_housing}
\end{figure}

\begin{figure}
\centering
\includegraphics[width=\linewidth]{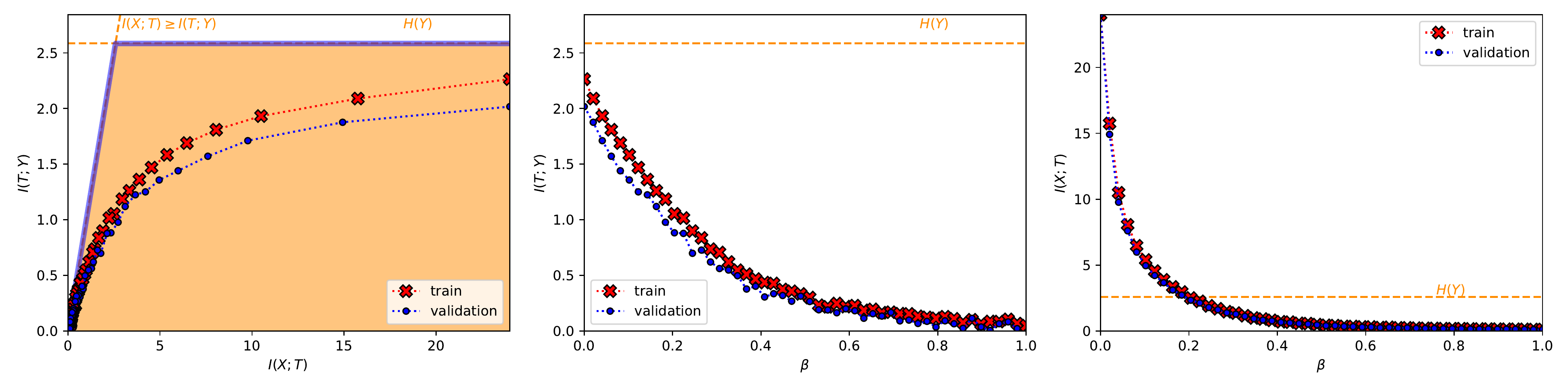}
\includegraphics[width=\linewidth]{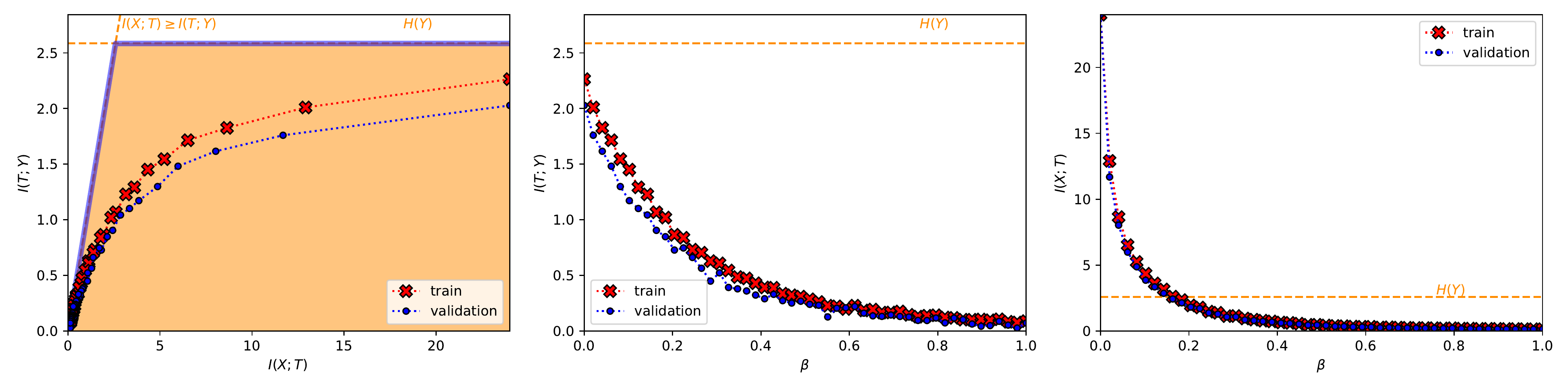}
\caption{{The top row shows the results for the normal IB Lagrangian, and the bottom row for the power IB Lagrangian with $\alpha = 0.1$, both in the TREC-6 dataset. In each row, from left to right it is shown (i) the information plane, where the region of possible solutions of the IB problem is shadowed in light orange and the information-theoretic limits are the dashed orange line; (ii) $I(T;Y)$ as a function of $\beta_u$; and (iii) the compression $I(X;T)$ as a function of $\beta_u$. In all plots, the red crosses joined by a dotted line represent the values computed with the training set and the blue dots the values computed with the validation set. Moreover, in all plots, it is indicated $H(Y) = \log_2(6)$. All values are shown in bits.}}
\label{fig:example_performance_trec}
\end{figure}

\section{Guidelines for selecting a proper function in the Convex IB Lagrangian}
\label{app:guidelines_on_choosing_proper_h}

When chossing the right $u$ function, it is important to find the right balance between avoiding value convergence and aiming for strong convexity. Practically, this balance is found by looking at how much faster $u$ grows w.r.t. the identity function. 

{When the aim is not to draw the IB curve but to find a specific level of performance, we can exploit the value convergence phenomenon in order to design a stable performance targeted $u$ function.}

\subsection{Avoiding value convergence}

In order to explain this issue we are going to use the example of classification on MNIST \citep{lecun1998gradient}, where $I(X;Y) = H(Y) = \log_2(10)$, and again the power and exponential IB Lagrangians.

If we use Proposition \ref{prop:bijective_mapping_beta_ixt} on both Lagrangians we obtain the bijective mapping between their Lagrange multipliers and a certain level of compression in the classification setting:

\begin{enumerate}
\item Power IB Lagrangian: $\beta_{\textnormal{pow}} = \left( (1+\alpha)I(X;T)^\alpha \right) ^{-1}$ and $I(X;T) = \left( (1+\alpha)\beta_{\textnormal{pow}} \right) ^{-\frac{1}{\alpha}}$.
\item Exponential IB Lagrangian: $\beta_{\textnormal{exp}} = \left( \eta \exp(\eta I(X;T)) \right)^{-1}$ and $I(X;T) = - \log(\eta \beta_{\textnormal{exp}})/\eta$.
\end{enumerate}

Hence, we can simply plot the curves of $I(X;T)$ vs. $\beta_u$ for different hyperparameters $\alpha$ and $\eta$ (see Figure \ref{fig:effect_beta_hyperparameters}). In this way we can observe how increasing the growth of the function (e.g., increasing $\alpha$ or $\eta$ in this case) too much provokes that many different values of $\beta_u$ converge to very similar values of $I(X;T)$. This is an issue both for drawing the curve (for obvious reasons) and for aiming for a specific performance level. Due to the nature of the estimation of the IB Lagrangian, the theoretical and practical value of $\beta_u$ that yield a specific $I(X;T)$ may vary slightly (see Figure \ref{fig:example_performance}). Then if we select a function with too high growth, a small change in $\beta_u$ can result in a big change in the performance obtained.

\begin{figure}
\centering
\includegraphics[width=\linewidth]{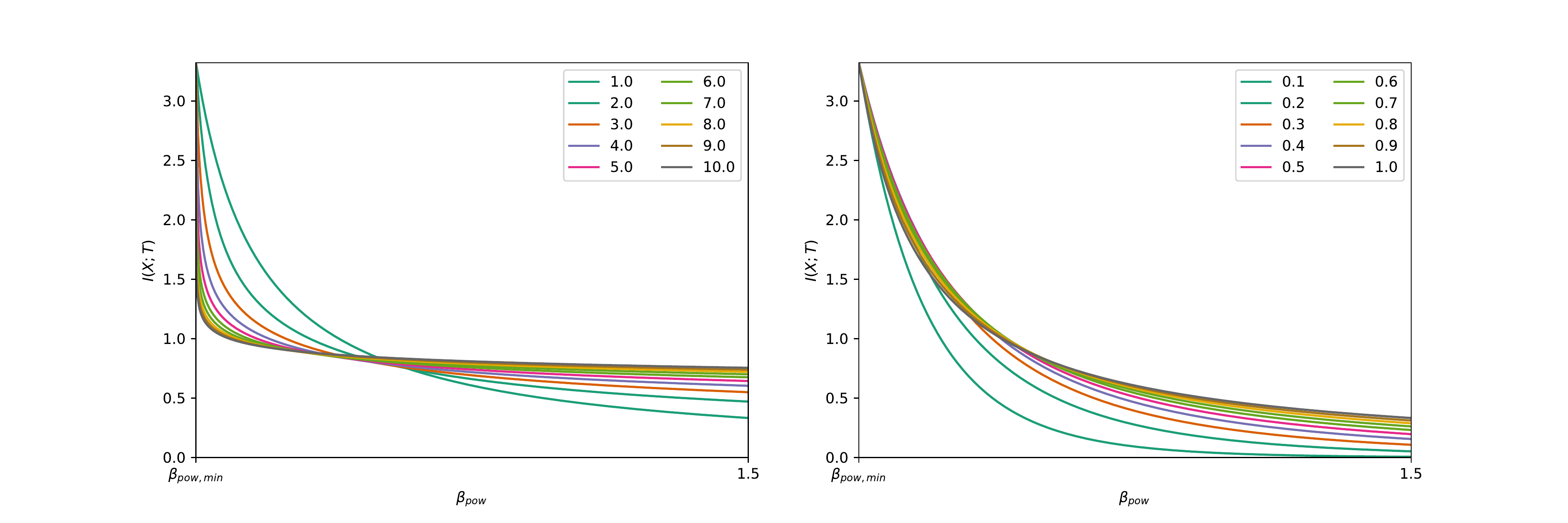} 
\includegraphics[width=\linewidth]{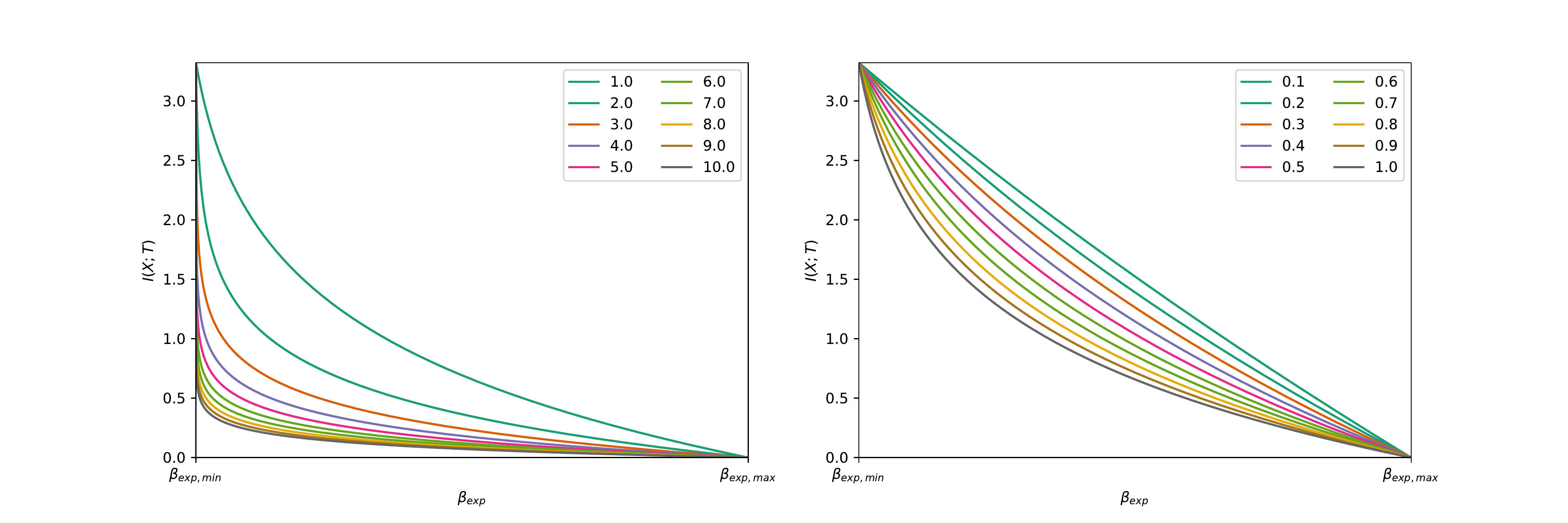}

\caption{Theoretical bijection between $I(X;T)$ and different $\alpha$ from $\beta_{u,\textnormal{min}}$ to 1.5 in the power IB Lagrangian (top), and different $\eta$ in the domain $B_u$ in the exponential IB Lagrangian (bottom).} 
\label{fig:effect_beta_hyperparameters}
\end{figure}

\subsection{Aiming for strong convexity}

\begin{Definition}[$\mu$-Strong convexity] If a function $f(r)$ is twice continuous differentiable and its domain is confined in the real line, then it is $\mu$-strong convex if $f''(r) \geq \mu \geq 0$  $\forall r$.
\end{Definition}

Experimentally, we observed when the growth of our function $u(r)$ is small in the domain of interest $r > 0$ the convex IB Lagrangian does not perform well {(see first row of Figures \ref{fig:example_performance_power} and \ref{fig:example_performance_exponential})}. Later we realized that this was closely related with the strength of the convexity of our function.

In Theorem \ref{th:ib_convex_lagrangians} we imposed the function $u$ to be strictly convex to enforce having a unique $\beta_u$ for each value of $I(X;T)$. Hence, since in practice we are not exactly computing the Lagrangian but an estimation of it (e.g., with the nonlinear IB \citep{kolchinsky2017nonlinear}) we require strong convexity in order to be able to explore the IB curve.

We now look at the second derivative of the power and exponential function: $u''(r) = (1 + \alpha) \alpha r ^{\alpha -1}$ and $u''(r) = \eta^2 \exp(\eta r)$ respectivelly. Here we see how both functions are inherently 0-strong convex for $r > 0$ and $\alpha, \eta > 0$. However, values of $\alpha < 1$ and $\eta < 1$ could lead to low $\mu$-strong convexity in certain domains of $r$. Particularly, the case of $\alpha < 1$ is dangerous because the function approaches 0-strong convexity as $r$ increases, so the power IB Lagrangian performs poorly when low $\alpha$ are used to find high performances.

\subsection{Exploiting value convergence}

{When the aim is not to draw or explore the IB curve, but to obtain a specific level of performance, the power or exponential IB Lagrangians aforementioned might not be the best choice due to the problems with value convergence or non-strong convexity. However, we can exploit the former in order to design a performance targeted $u$ function.}

{For instance, if we look at Figure \ref{fig:effect_beta_hyperparameters} we can see how a modification of the exponential IB Lagrangian could result in such a function. More precisely, a shifted exponential $u(r) = \exp(\eta (r - r^*))$, with $\eta > 0$ sufficiently large, converges to the compression level $r^*$. We can see this more clearly if we consider the shifted exponential IB Lagrangian $\mathcal{L}_{\textnormal{IB,sh-exp}}(T;\beta_{\textnormal{sh-exp}},\eta,r^*) = I(T;Y) - \beta_{\textnormal{sh-exp}} \exp(\eta (I(X;T) - r^*))$, since then the application of Proposition \ref{prop:bijective_mapping_beta_ixt} results on $I(X;T) = - \log(\eta \beta_{\textnormal{sh-exp}} / f'_{\textnormal{IB}}(I(X;T)))/\eta + r^*$, where $f'_{\textnormal{IB}}(I(X;T))$ is the derivative of $f_{\textnormal{IB}}$ evaluated at $I(X;T)$. We know $f_{\textnormal{IB}}' = 1$ in deterministic scenarios (Theorem \ref{th:ib_curve_piecewise_linear}) and that $f'_{\textnormal{IB}} < 1$ otherwise (see, e.g., \citep{wu2019learnability}). Then, for large enough $\eta$, $I(X;T) \approx r^*$ regardless of the value of $f'_{\textnormal{IB}}$.}

{For instance, if we consider a deterministic scenario like the MNIST dataset \citep{lecun1998gradient} with $I(X;Y) = H(Y) = \log_2(10)$, for $\eta = 200$ and $r^* = 2$ the range of the Lagrange multipliers that allow the exploration of the IB curve, according to Corollary \ref{cor:domain_conv_ib_lagrange}, is $\beta_\textnormal{sh-exp} \in [7.54 \cdot 10^{-178},2.61 \cdot 10^{171}]$. Furthermore, $I(X;T)$ is close to 2 for many values of $\beta_{\textnormal{sh-exp}}$. For instance, $I(X;T) = 1.974$ for $\beta_{\textnormal{sh-exp}}= 1$ and $I(X;T) = 1.963$ for $\beta_{\textnormal{sh-exp}} = 8$. This ensures a stability in the performance level obtained so that small changes in the choice of $\beta_{\textnormal{sh-exp}}$ do not result in significant changes on the performance (e.g., see top row from Figure \ref{fig:example_performance_value_convergence}).}

{If we now consider a stochastic scenario like the TREC-6 dataset \citep{li2002learning} with $H(Y) = \log_2(6)$, for $\eta = 200$ and $r^* = 16$ the range of the Lagrange multipliers that allow the IB curve, according to Corollary \ref{cor:bound_domain}, is $\beta_\textnormal{sh-exp} \in [0, 2.76 (\inf_{\Omega_x \subset \mathcal{X}} \lbrace \beta_0(\Omega_x)  \rbrace)^{-1} \cdot 10^{1287}]$, where $\beta_0$ and $\Omega_x$ are defined as in \citep{wu2019learnability}. Then, unless $(\inf_{\Omega_x \subset \mathcal{X}} \lbrace \beta_0(\Omega_x)  \rbrace)^{-1}$ is of the order of $10^{-1287}$, the range of possible betas is wide. Moreover, $I(X;T)$ is close to $16$ for many values of $\beta_{\textnormal{sh-exp}}$. For example, $I(X;T) = 15.939$ if $f_{\textnormal{IB}}' = 0.001$ at that point and $I(X;T) = 15.973$ if $f_{\textnormal{IB}}' = 0.9$ for $\beta_{\textnormal{sh-exp}} = 1$; and $I(X;T) = 15.929$ if $f_{\textnormal{IB}}' = 0.001$ at that point and $I(X;T) = 15.963$ if $f_{\textnormal{IB}}' = 0.9$ for $\beta_{\textnormal{sh-exp}} = 8$. Hence, as in the deterministic scenario, the performance level obtained is stable with changes in the choice of $\beta_{\textnormal{sh-exp}}$ (e.g., see bottom row from Figure \ref{fig:example_performance_value_convergence}).}

\end{document}

%% file: math_commands.tex
\usepackage{amsmath,amsfonts,bm}

\def\eqref#1{equation~\ref{#1}}

\def\floor#1{\lfloor #1 \rfloor}
\def\1{\bm{1}}

\DeclareMathAlphabet{\mathsfit}{\encodingdefault}{\sfdefault}{m}{sl}
\SetMathAlphabet{\mathsfit}{bold}{\encodingdefault}{\sfdefault}{bx}{n}

\DeclareMathOperator*{\argmax}{arg\,max}